\documentclass[letterpaper]{article} % DO NOT CHANGE THIS
\usepackage{aaai24}  % DO NOT CHANGE THIS
\usepackage{times}  % DO NOT CHANGE THIS
\usepackage{helvet}  % DO NOT CHANGE THIS
\usepackage{courier}  % DO NOT CHANGE THIS
\usepackage[hyphens]{url}  % DO NOT CHANGE THIS
\usepackage{graphicx} % DO NOT CHANGE THIS
\usepackage{natbib}  % DO NOT CHANGE THIS AND DO NOT ADD ANY OPTIONS TO IT
\usepackage{caption} % DO NOT CHANGE THIS AND DO NOT ADD ANY OPTIONS TO IT
\usepackage{algorithm}
\usepackage{algorithmic}

\usepackage{amsmath}
\usepackage{subcaption}
\usepackage{threeparttable}
\usepackage{diagbox}

\usepackage{epsfig,epsf,fancybox}
\usepackage{mathrsfs}
\usepackage{amssymb}
\usepackage{color}
\usepackage{multirow}
\usepackage{paralist}
\usepackage{verbatim}
\usepackage{galois}
\usepackage{boxedminipage}
\usepackage{accents}
\usepackage{stmaryrd}
\usepackage[bottom]{footmisc}
\definecolor{light-gray}{gray}{0.85}
\usepackage{colortbl}
\usepackage{makecell}
\usepackage{hhline}
\usepackage{mathtools}
\usepackage{xcolor}         % colors
\usepackage{bbold}
\usepackage[utf8]{inputenc} % allow utf-8 input
\usepackage{url}            % simple URL typesetting
\usepackage{booktabs}       % professional-quality tables
\usepackage{amsfonts}       % blackboard math symbols
\usepackage{nicefrac}       % compact symbols for 1/2, etc.
\usepackage{microtype}      % microtypography
\usepackage{MnSymbol}

\newtheorem{theorem}{Theorem}[section]

\newcommand{\Lamb}{{\boldsymbol \lambda}}

\newcommand{\g}{\mathbf g}

\newcommand{\mc}{\mathcal}
\newcommand{\mbb}{\mathbb}
\newcommand{\mb}{\mathbf}

\newcommand{\ba}{\begin{array}}
\newcommand{\ea}{\end{array}}

\usepackage{newfloat}
\usepackage{listings}
\DeclareCaptionStyle{ruled}{labelfont=normalfont,labelsep=colon,strut=off} % DO NOT CHANGE THIS
\floatstyle{ruled}
\newfloat{listing}{tb}{lst}{}
\floatname{listing}{Listing}
\title{Balance Reward and Safety Optimization for Safe Reinforcement Learning: A Perspective of Gradient Manipulation}
\author {
    Shangding Gu\textsuperscript{\rm 1},
    Bilgehan Sel\textsuperscript{\rm 2},
    Yuhao Ding\textsuperscript{\rm 3},
    Lu Wang\textsuperscript{\rm 4},
    Qingwei Lin\textsuperscript{\rm 4},
    Ming Jin\textsuperscript{\rm 2}\thanks{Equally advise. Corresponding authors: Ming Jin (\textit{jinming@vt.edu}) and Shangding Gu (\textit{shangding.gu@tum.de}). \\ Copyright © 2024, Association for the Advancement of Artificial
Intelligence (www.aaai.org). All rights reserved.},
    Alois Knoll\textsuperscript{\rm 1}$^{*}$
}
\affiliations {
    \textsuperscript{\rm 1}Technical University of Munich\\
    \textsuperscript{\rm 2} Virginia Tech\\
    \textsuperscript{\rm 3}University of California, Berkeley\\
    \textsuperscript{\rm 4}Microsoft Research\\
}
\usepackage{bibentry}
\begin{document}

\maketitle

\begin{abstract}
Ensuring the safety of Reinforcement Learning (RL) is crucial for its deployment in real-world applications. Nevertheless, managing the trade-off between reward and safety during exploration presents a significant challenge. Improving reward performance through policy adjustments may adversely affect safety performance. In this study, we aim to address this conflicting relation by leveraging the theory of gradient manipulation. Initially, we analyze the conflict between reward and safety gradients. Subsequently, we tackle the balance between reward and safety optimization by proposing a soft switching policy optimization method, for which we provide convergence analysis. Based on our theoretical examination, we provide a safe RL framework to overcome the aforementioned challenge, and we develop a Safety-MuJoCo Benchmark to assess the performance of safe RL algorithms. Finally, we evaluate the effectiveness of our method on the Safety-MuJoCo Benchmark and a popular safe RL benchmark, Omnisafe. Experimental results demonstrate that our algorithms outperform several state-of-the-art baselines in terms of balancing reward and safety optimization.
\end{abstract}

\section{Introduction}
\label{section:introduction}

Reinforcement Learning (RL) has demonstrated remarkable performance in various scenarios \cite{gu2022review}, such as the game of Go \cite{silver2016mastering}, autonomous driving \cite{kiran2021deep, gu2022constrained}, and robotics \cite{kober2013reinforcement, gu2023safe}. However, the majority of RL methods are restricted to simulation environments due to safety concerns associated with deploying RL in real-world settings. To address this issue, numerous safe RL methods have been proposed to tackle the safety challenge.

For instance, Constrained Policy Optimization (CPO)~\cite{achiam2017constrained} is developed to ensure reward monotonic improvement while maintaining safety. PPO Lagrangian and TRPO Lagrangian methods \cite{ray2019benchmarking} are introduced to address the balance between reward and safety performance by employing Lagrangian optimization. Additionally, safe exploration methods based on a Gaussian Process (GP)~\cite{sui2015safe} are developed to guarantee exploration safety by utilizing a GP to model the exploration safety of the state space. However, these methods may not effectively resolve the conflict reward and cost gradients, and balance reward and safety optimization.
A key question that is raised in this domain is: How can we handle the balance between reward and safety optimization? 

In this research, we aim to address the key question by leveraging the theory of gradient manipulation \cite{yu2020gradient, chen2021generalized, liu2021conflict, zhou2022convergence}, wherein we conduct a detailed examination of the changes in gradients associated with reward and safety. Based on our theoretical analysis, we propose the projection constraint-rectified policy optimization (PCRPO) method, designed to alleviate the conflict between reward and safety optimization while maintaining a balance between their optimization levels, in which soft switching policy optimization through gradient manipulation is proposed and a slack technique is introduced for adjusting the emphasis on safety optimization. Particularly, we evaluate the effectiveness of our method on a multitude of challenging tasks, and conduct ablation experiments to thoroughly examine the performance of our method. The empirical findings suggest that our approach outperforms strong baselines concerning the balance between reward maximization and safety preservation.

The present study offers several significant contributions to the field, which are enumerated as follows: (1)
The introduction of a novel problem concerning safe RL involving conflicts between reward and cost gradients. (2)
The development of a safe RL framework employing soft switching via gradient manipulation. (3)
The establishment of a new benchmark for Safe RL evaluation, designed to assess the performance of safe RL algorithms; (4)
The demonstration that the practical algorithms proposed in this study surpass existing state-of-the-art baselines with respect to both reward and safety performance.

\section{Related Work}
\label{section:related_work}

In recent years, numerous safe RL methods have been proposed to ensure RL safety~\cite{gu2022review, gu2023human}. These safe RL methods can be briefly categorized into three main groups: (1) Control theory-based safe RL: These methods leverage principles from control theory, such as model predictive control and Lyapunov functions, to ensure that the agent operates within safety constraints while learning optimal behavior~\cite{koller2018learning}. (2) Formal methods-based safe RL: These approaches employ formal verification and synthesis techniques, such as temporal logic, to guarantee that the learned policies satisfy safety specifications~\cite{fulton2018safe}. (3) Constrained optimization-based safe RL: These methods focus on optimizing the agent's behavior while adhering to safety constraints. Techniques like constrained policy optimization and Lagrangian relaxation are used to ensure that the RL algorithm respects the safety limits during learning~\cite{brunke2022safe}.

Specifically, from the control theory perspective, Lyapunov functions are employed to ensure learning safety by constraining the action space of exploration~\cite{chow2018lyapunov, chow2019lyapunov}. Although Lyapunov function-based methods can demonstrate good performance in ensuring learning safety, defining specific Lyapunov functions requires a system model, and it is usually challenging to find a function that can handle general safe RL problems. From the formal methods perspective, some methods based on formal techniques are proposed to guarantee RL safety. For example, temporal logic verification is used to verify safe actions during exploration~\cite{li2019temporal}. Such methods can rigorously ensure learning safety. However, external knowledge is needed to define the safe state and action space, which may be difficult to deploy in real-world RL applications.

Compared to the aforementioned methods, constrained optimization-based safe RL methods have gained considerable attention due to their relative maturity and broad applicability. One branch of constrained optimization-based safe RL methods encompasses primal-dual methods~\cite{boyd2004convex}. A notable method within this branch is CPO, which employs TRPO~\cite{schulman2015trust} in constrained optimization and can nearly guarantee hard constraints via a line search method~\cite{nocedal1998combining}. PPO-Lagrangian, another representative primal-dual optimization method~\cite{zhou2023gradient, calian2020balancing}, is developed based on Lagrangian optimization and dynamically adjusts the Lagrangian multiplier in response to safety violations. Following CPO and PPO-Lagrangian, recent state-of-the-art baselines, such as PCPO~\cite{yang2020projection} and CUP~\cite{yang2022constrained},  are proposed to ensure learning safety. 
Another branch of constrained optimization-based safe RL methods consists of primal methods~\cite{boyd2004convex}. CRPO~\cite{xu2021crpo}, a representative method for primal optimization, directly enhances reward performance while ensuring learning safety within the primal problem. In contrast to primal-dual-based methods, primal-based methods offer ease of implementation and are not burdened by hyperparameter tuning issues related to dual variables. Moreover, primal-based safe RL methods do not necessitate feasible initialization. However, poor initialization can adversely affect the performance of primal-dual optimization-based methods~\cite{xu2021crpo}.

The methods mentioned above do not explicitly analyze and address the gradient conflicts between reward and cost optimization. This oversight can lead to significantly negative effects on safe RL performance. In contrast to previous works, our proposed method, which is based on primal optimization, necessitates only gradients from the objective and the costs to ensure safe exploration. This is a key difference from other methods like CRPO, where gradient conflicts may lead to unsafe exploration and wasted samples during training. By focusing on resolving these gradient conflicts, our approach aims to provide a more effective solution for safe RL applications.

\section{Problem Formulation}
\label{section:problem_formulation}

\paragraph{Markov Decision Processes}
An infinite-horizon Markov Decision Process  MDP$(\mc{S},\mc{A},P,r,\gamma)$  is specified by: a state space $\mc{S}$; an action space $\mc{A}$; a transition dynamics $P: \mc{S}\times \mc{A}\times\mc{S} \rightarrow [0,1]$, where $P(s^\prime \vert s,a)$ is the probability of transition from state $s$ to state $s^\prime$ when action $a$ is taken; a reward function $r: \mc{S}\times \mc{A} \rightarrow \mathbb{R}$, where $r(s,a)$ is the instantaneous reward when taking action $a$ in state $s$; a discount factor $\gamma\in[0,1)$. 
A policy $\pi: \mc{S}\rightarrow \Delta (\mc{A})$ represents that the decision rule the agent uses, i.e. the agent takes action $a$ with probability $\pi(a\vert s)$ in state $s$. 
Given a policy $\pi$, the value function $V^\pi : \mc{S} \rightarrow \mathbb{R}$ is defined to characterize the discounted sum of the rewards earned under $\pi$, i.e.
\begin{equation}
V_r^{\pi}(s):=\mbb{E}\left[\sum_{t=0}^{\infty} \gamma^{t} r\left(s_{t}, a_{t}\right) \bigg| \pi, s_{0}=s\right],\ \forall s\in\mc{S}
\end{equation}
where the expectation is taken over all possible trajectories, in which $a_t \sim \pi(\cdot \vert s_t)$ and $s_{t+1} \sim P(\cdot \vert s_t, a_t)$.
When the initial state is sampled from some distribution $\rho$, we slightly abuse the notation and define the value function as
\begin{equation}
V_r^{\pi}(\rho):=\mbb{E}_{s\sim \rho}\left[V^{\pi}(s)\right].
\end{equation}

The action-value function (or Q-function) $Q_r^\pi:\mc{S}\times \mc{A}\rightarrow \mbb{R}$ under policy $\pi$ is defined as
\begin{equation}
Q_r^{\pi}(s, a)=\mathbb{E}\left[\sum_{t=0}^{\infty} \gamma^{t} r\left(s_{t}, a_{t}\right) \bigg| \pi, s_{0}=s, a_{0}=a\right],
\end{equation}
which can be interpreted as the expected total reward with an initial state $s_0 = s$ and an initial action $a_0 = a$.

\paragraph{Constrained MDP}
In a Constrained Markov Decision Process CMDP$(\mc{S},\mc{A},P,r,\mb{c},\mb{b},\gamma)$, besides the reward function $r$, we have a cost function $\mb{c} = (c_1,\dots,c_n): \mc{S}\times \mc{A} \rightarrow \mathbb{R}^n $ and a threshold $\mb{b}=(b_1,\dots,b_n)\in \mathbb{R}^n$.
In the safety-critical environments, the agent aims at maximizing the expected (discounted) cumulative reward for a given initial distribution $\rho$ while satisfying constraints on the expected (discounted) cumulative cost, i.e.,
\begin{equation}\label{eq:cmdp}
\begin{aligned}
\max_{\pi\in \Pi}\  V^\pi_r(\rho), \ \text{ s.t. } V^\pi_{c_i}(\rho) \leq b_i, \ \forall i =1,\ldots,n.
\end{aligned}
\end{equation}
where the expectation is taken over all possible trajectories, and $ V^\pi_r(\rho)$ and $ V^\pi_{c_i}(\rho)$ denote the value function corresponding to the reward and cost functions, respectively.

\paragraph{Primal vs Primal-dual Approaches}
The current safe RL methods can be generally categorized into the \textbf{primal} and \textbf{primal-dual} approaches. In \textbf{primal-dual} optimization, the
primal-dual approaches convert the constrained problem  \eqref{eq:cmdp} into an unconstrained one by augmenting the objective with a sum of constraints
weighted by their corresponding dual variables $\Lamb$. The associated Lagrangian function $L(\pi,\Lamb)$ is defined as:
\begin{align}
L(\pi,\Lamb) &:= V_{r}^\pi(\rho) - \Lamb^T\left(V_\mb{c}^\pi (\rho)-\mb{b}\right), \label{def:lagrangian}
\end{align}
where $\Lamb \in \mathbb{R}_+^{n}$. The safe policy is learned from  
applying a certain policy optimization update such as (natural) policy
gradient alternatively with a gradient descent type update
for the dual variables: $\pi_{t+1} = \pi_t +\eta_1 \widetilde{\nabla_\pi L} (\pi_t,\Lamb_t),
\Lamb_{t+1} = \mc{P}_U\left(\Lamb_t - \eta_2 \left(V_\mb{c}^\pi (\rho)-\mb{b} \right)    \right),\quad \text{for } t=0,1,2,\dots,$ where $\eta_1 >0$, $\eta_2 >0$ are step-sizes, $\widetilde{\nabla_\pi L}$ can be the policy gradient or its variants,  and the dual feasible region $U:= [0,C_0]$ is an interval that contains $\Lamb^\star$.

In \textbf{primal} optimization, the necessity for dual variables is eliminated, enabling the immediate optimization of rewards and costs. Such approaches, exemplified by CRPO~\cite{xu2021crpo}, have demonstrated superior outcomes compared to conventional primal-dual techniques with guaranteed convergence. Nevertheless, when transitioning between reward and cost optimization, conflicting relationships may arise between reward gradients $\boldsymbol{g_{r}}$ and cost gradients $\boldsymbol{g_{c_i}}$. This conflict $f_{rc_i}$ has the potential to negatively impact the efficacy of primal methods in terms of both reward and safety performance.

As depicted in Figure~\ref{fig:primal-crpo-problem}, the reward gradient is represented by $\g_r$, while the cost gradient is denoted by $\g_c$. Additionally, $\g_c^-$ signifies the projection of the cost gradient $\g_c$ onto the plane of the reward gradient $\g_r$, and $\g_r^-$ refers to the projection of the reward gradient $\g_r$ onto the plane of the cost gradient $\g_c$. During the primal optimization, a transition from cost optimization to reward optimization occurs. In this scenario, the cost gradient optimization process adversely impacts the reward optimization process ($p_2$). Consequently, the current gradient is expressed as $\g = \g_r - \g_c^-$. Conversely, when switching from reward optimization to cost optimization ($p_1$), the current gradient is given by $\g = \g_c - \g_r^-$.

  \begin{figure}[tb!]
 \centering
 % \subcaptionbox{}
 {
\includegraphics[width=0.3\linewidth, angle=-90]{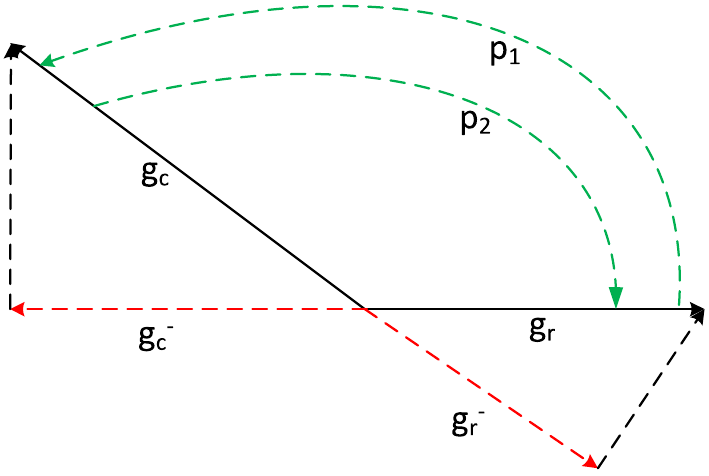}   
}
 	\caption{\normalsize Conflicts between reward and cost optimization.
 	} 
  \label{fig:primal-crpo-problem}
 \end{figure} 

Consequently, it is crucial to ascertain a method for balancing the relation between rewards and costs while simultaneously mitigating the adverse effects of conflicting gradients and minimizing the optimization oscillation on the overall performance. To satisfy the above requirements, we optimize reward and safety performance while minimizing the deviation between reward and cost gradients. This helps prevent conflicting gradients and oscillations during the optimization process.

% Consequently, it is crucial to ascertain a method for delicately balancing the relation between rewards and costs while simultaneously mitigating the adverse effects of conflicting gradients and minimizing the optimization oscillation on the overall performance of the approach. To satisfy the above requirements, we a safe RL problem is rewritten as Equation (\ref{eq:conflict-safe-rl-problem}), where $f_{rc_i}$ denotes the deviation between gradients, which could result in conflicting gradients and oscillations in the optimization process.
% \begin{equation}
% \begin{aligned}
% \label{eq:conflict-safe-rl-problem}
% \max_{\pi\in \Pi} &\  V^\pi_r(\rho), \\  \text{ s.t. } & V^\pi_{c_i}(\rho) \leq b_i, \ \forall i =1,\ldots,n. \\
% & \min_{\pi\in \Pi}  f_{rc_i} (\boldsymbol{\g_{r}}, \boldsymbol{\g_{c_i}}).
% \end{aligned}
% \end{equation}

\section{Method}
\label{section:method}

In order to tackle the safe RL problem, as illustrated in Equation~(\ref{eq:cmdp}), it is imperative to first address the conflict between reward and cost gradients. To this end, we introduce a novel soft switching optimization solution that employs gradient manipulation to achieve a balanced relationship between these gradients. This approach incorporates a slack mechanism designed to smoothly optimize both reward and cost. Subsequently, we analyze the gradient change by soft switching. Then, the convergence analysis is provided. Lastly, we present a safe RL framework through gradient manipulation based on primal optimization, and a practical algorithm that effectively facilitates the implementation of our proposed method in real-world scenarios.

% Subsequently, we examine the gradient equilibrium between reward and cost optimization through the lens of game theory, wherein each gradient is characterized as an individual agent. Then, the convergence analysis is provided. Lastly, we present a practical algorithm that effectively facilitates the implementation of our proposed method in real-world scenarios.
\subsection{Soft Switching through Gradient Manipulation}
By leveraging the gradient manipulation,  we effectively regulate the switching transitions and minimize the oscillations between reward and cost optimization. The objective of soft switching in this context is to enhance the overall efficiency, performance, and reliability of the algorithms, while simultaneously reducing the deviation of gradients between reward and cost components.
As illustrated in Figure~\ref{fig:finish-pcrpo-projection-case-one-two}, the cost gradient of constraint $i$ is denoted by $\g_{c_i}$. For the sake of simplicity, we represent it as $\g_c$. The projection gradient of $\g_c$ on the normal plane of gradient $\g_r$ is given by $\g_{c_i}^+$, while the projection gradient of $\g_r$ on the normal plane of gradient $\g_{c_i}$ is represented by $\g_{r}^+$. The angle between gradients $g_r$ and $g_c$ is denoted by $\theta$.

% During the process of gradient manipulation, if $\theta$ exceeds $90^{\circ}$, the gradient projection is employed, as demonstrated in Equation (\ref{eq:two-projection-gradient}). In contrast, when $\theta$ is less than or equal to $90^{\circ}$, Equation~(\ref{eq:two-projection-gradient-average}) is leveraged to manage gradient manipulation. We try to reduce the optimization oscillations by reducing deviation between reward and cost gradient, and find a gradient that can better satisfy safety constraints and improve reward performance. 

During the gradient manipulation process, the gradient projection is employed if the angle $\theta$ exceeds $90^{\circ}$, as demonstrated in Equation (\ref{eq:two-projection-gradient}) (A simplified illustration of this scenario can be observed in Figure~\ref{fig:finish-pcrpo-projection-case-one-two} (a)), where $\beta^+_r$ and $\beta^+_c$ denote the weights of gradient $\g_r^+$ and gradient $\g_c^+$, respectively. Conversely, when the angle $\theta$ is less than or equal to $90^{\circ}$, Equation~(\ref{eq:two-projection-gradient-average}) is leveraged to handle gradient manipulation (A simplified illustration of this scenario can be observed in Figure~\ref{fig:finish-pcrpo-projection-case-one-two} (b)),  where $\beta_r$ and $\beta_c$ denote the weights of gradient $\g_r$ and gradient $\g_c$, respectively. Our approach aims to minimize optimization oscillations by reducing the deviation between reward and cost gradients, $f_{rc} = f(\boldsymbol{\g_{r}}, \boldsymbol{\g_{c}}) = \theta$, particularly for conflicting gradients between reward and cost optimization.  This finally allows us to identify a gradient that can effectively satisfy safety constraints while simultaneously enhancing reward performance. In next section, we will analyze how gradient change with soft switching.

% reduce the  to ensure consistency
% As shown in Figure~(\ref{fig:finish-pcrpo-projection}), $\g_{c_i}$ denotes the cost gradient of constraint $i$, for simplicity, we write is as $\g_c$, $\g_{c_i}^-$ denotes the projection gradient of $\g_c$ on the normal plane of gradient $\g_r$, $\g_{r}^-$ denotes the projection gradient of $\g_r$ on the normal plane of gradient $\g_{c_i}$. $\theta$ can be the angle of gradient $g_r$ and $g_c$. During gradient manipulation, when $\theta$ is more than $90^{\circ}$, we will take the gradient projection, as shown in Equation (\ref{eq:two-projection-gradient}), otherwise, we will leverage Equation~(\ref{eq:two-projection-gradient-average}) to handle gradient manipulation.

% \begin{align}
% \label{eq:conflict-equation}
% f_{rc} = f(\boldsymbol{\g_{r}}, \boldsymbol{\g_{c}}) = \theta
%     % f_{rc_i} & = f(\boldsymbol{\g_{r}}, \boldsymbol{\g_{c}}) \nonumber\\
%     % &= cos~\theta \\
%     % & = cos (\boldsymbol{\g_r}, \boldsymbol{\g_c}). \nonumber
% \end{align}

\begin{align}
\label{eq:initial-gradient-projection}
    \g_r^+ = \g_r - \frac{\g_r \cdot \g_{c}}{|\g_{c}|^2} \g_{c}, \ \ \g_c^+ = \g_{c} - \frac{\g_c \cdot \g_{r}}{|\g_{r}|^2} \g_{r},
\end{align}
% \begin{align*}
% % \label{eq:projection-gradient}
%     \g_c^+ = \g_{c} - \frac{\g_c \cdot \g_{r}}{|\g_{r}|^2} \g_{r},
% \end{align*}
\begin{align}
\label{eq:two-projection-gradient}
    \g = \beta^+_r \g^+_{r} + \beta^+_c \g^+_{c}, 
    %\frac{{\g_{r}^+} + {\g_{c}^+}}{2},
\end{align}
\begin{align}
\label{eq:two-projection-gradient-average}
    \g = \beta_r \g_{r} + \beta_c \g_{c}. 
    %\frac{{\g_{r}} + {\g_{c}}}{\beta}.
\end{align}
% \vspace{-13pt}

  \begin{figure}[tb!]
 \centering
%  \subcaptionbox{}
%  {
% \includegraphics[width=0.67\linewidth, angle=0]{files/figures/method/PCRPO-projection.pdf}   
% }
 \subcaptionbox{}
 {
\includegraphics[width=0.5\linewidth, angle=0]{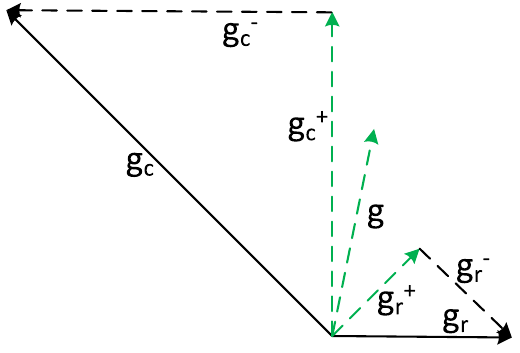}   
}
\subcaptionbox{}
 {
\includegraphics[width=0.28\linewidth, angle=0]{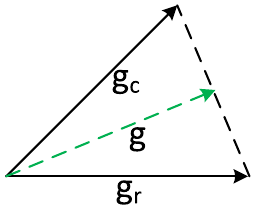}   
}
    % \vspace{-5pt}
 	\caption{\normalsize Soft switching through gradient manipulation.
 	} 
  \label{fig:finish-pcrpo-projection-case-one-two}
  % \vspace{-15pt}
 \end{figure}

 \begin{comment}
     \begin{figure}[htbp!]
 \centering
 % \subcaptionbox{}
 {
\includegraphics[width=0.5\linewidth, angle=0]{files/figures/method/PCRPO-projection.pdf}   
}
    \vspace{-0pt}
 	\caption{\normalsize Case 1: Soft switching through gradient manipulation.
 	} 
  \label{fig:finish-pcrpo-projection-case-one}
 \end{figure} 
  \begin{figure}[htbp!]
 \centering
 % \subcaptionbox{}
 {
\includegraphics[width=0.4\linewidth, angle=0]{files/figures/method/PCRPO-average.pdf}   
}
    \vspace{-0pt}
 	\caption{\normalsize Case 2: Soft switching through gradient manipulation.
 	} 
  \label{fig:finish-pcrpo-projection-case-two}
 \end{figure} 
 \end{comment}
 % \vspace{-10pt}
\subsection{Gradient Analysis with Soft Switching}
\label{subsec:gradient-analysis-with-soft-switching}
% \textbf{Analysis }

% When $\theta > 90^\circ$, if we still take Equation~(\ref{eq:two-projection-gradient-average}) to handle the policy gradient, as shown in Figure~\ref{fig:finish-pcrpo-projection-case-one-analysis}, the gradient would be represented as $\g^-$ which is denoted as the red arrow; if we take Equation~(\ref{eq:two-projection-gradient}), the gradient can be represented as $\g$, which is denoted as green dash line in Figure~\ref{fig:finish-pcrpo-projection-case-one-analysis}. By leveraging the property of geometry, when $\theta > 90^\circ$, the projection of gradient $\g_r$ on the normal plane of gradient $\g_c$ is $\g_r^+$, the projection of gradient $\g_c$ on the normal plane of gradient $\g_r$ is $\g_c^+$, we can observe that, under the condition, the following property holds, 

% \subsubsection{$\theta > 90^\circ$:}
In this analysis, for the purpose of simplification, we consider $\beta_r^+$, $\beta_c^+$, $\beta_r$ and $\beta_c$ to be equal to 0.5. Other cases follow a similar pattern and can be easily proven based on our analytical framework. In instances where $\theta \geq 90^\circ$, we have two strategies to handle the conflict gradients between reward and safety optimization. The first strategy is to leverage Equation~(\ref{eq:two-projection-gradient-average}) to address the policy gradient, as demonstrated in Figure~\ref{fig:finish-pcrpo-projection-case-one-two-analysis} (a), results in the gradient being represented by $\g^-$, which is indicated by the red arrow. The second strategy is to employ Equation~(\ref{eq:two-projection-gradient}) that allows the gradient to be depicted as $\g$, denoted by the green dashed line in Figure~\ref{fig:finish-pcrpo-projection-case-one-two-analysis} (a). To assess which gradient manipulation is better, we provide the following analysis for this instance. Specifically, by capitalizing on geometric properties, it can be observed that, when $\theta \geq 90^\circ$, the projection of gradient $\g_r$ on the normal plane of gradient $\g_c$ is $\g_r^+$, and the projection of gradient $\g_c$ on the normal plane of gradient $\g_r$ is $\g_c^+$. Consequently, under such conditions, the following property is maintained:

% \begin{align}
%     \frac{{\g_{r}^+} + {\g_{c}^+}}{2} \geq \frac{{\g_{r}} + {\g_{c}}}{2}
% \end{align}

\begin{equation}
  \begin{aligned}
    \g^- = \frac{{\g_{r}} + {\g_{c}}}{2}
    = \frac{\frac{\g_r}{\|\g_r\|}\left(\|\g_{r}\|\right) + \frac{\g_c}{\|\g_c\|}\left(\|\g_{c}\|\right)}{2},
    % \\
    % =& \frac{\frac{\g_r}{\|\g_r\|}\left(\|\g_{r}\| - \|g^-_c\|\right) + \frac{\g_c}{\|\g_c\|}\left(\|\g_{c}\| - \|\g_r^-\|\right)}{2}
\end{aligned}  
\end{equation}

\begin{equation}
  \begin{aligned}
  \label{eq:cos-vector-gradient}
    \cos \left(\theta\right)=\frac{\mathbf{g}_{\mathbf{r}} \cdot \mathbf{g}_{\mathbf{c}}}{\left\|\mathbf{g}_{\mathbf{r}}\right\|\left\|\mathbf{g}_{\mathbf{c}}\right\|},
\end{aligned}  
\end{equation}
with Equation~(\ref{eq:initial-gradient-projection}) and Equation (\ref{eq:cos-vector-gradient}), we can observe, 
\begin{equation}
  \begin{aligned}  
    &\resizebox{!}{0.7cm}{$\g = \frac{{\g_{r}^+} + {\g_{c}^+}}{2}
    =\frac{\left(\g_r - \frac{\g_r \cdot \g_{c}}{\|\g_{c}\|^2} \g_{c}\right) + \left(\g_{c} - \frac{\g_c \cdot \g_{r}}{\|\g_{r}\|^2} \g_{r}\right)}{2}$}
    \\
    = &\resizebox{!}{0.67cm}{$\frac{\left(\frac{\g_r}{\|\g_r\|} \|\g_r\| - \frac{cos(\theta)\|\g_r\| \|\g_{c}\|}{\|\g_{c}\|^2} \g_{c}\right) + \left(\frac{\g_c}{\|\g_c\|}\|\g_{c}\| - \frac{cos(\theta)\|\g_c\| \|\g_{r}\|}{\|\g_{r}\|^2} \g_{r}\right)}{2}$}.
\end{aligned}  
\end{equation}
Under the condition of $\theta \geq 90^\circ, cos(\theta) \leq 0$, we can observe,
\begin{align}
&\left( - \frac{cos(\theta)\|\g_r\| \|\g_{c}\|}{\|\g_{c}\|^2} \right) + \left( - \frac{cos(\theta)\|\g_c\| \|\g_{r}\|}{\|\g_{r}\|^2} \right) \geq 0 \nonumber\\
% &\geq -\frac{\g_r}{\|\g_r\|} \|\g_c\|  -\frac{\g_c}{\|\g_c\|} \|\g_r\| \nonumber\\
&\Longrightarrow
    \|\g\| \geq \|\g^-\|.
\end{align}

% Thus, when $\theta \geq 90^\circ$, the strategy of Equation (\ref{eq:two-projection-gradient}) is better than the stratgy of Equation (\ref{eq:two-projection-gradient-average}) to handle the deviation of gradients. which means the second strategy can avoid the gradient degradation well while handling the conflict gradients. 
% Particularly, the gradient $\g$ can be as an \textbf{Equilibrium Gradient} that will  balance reward and safety optimization well.

Hence, when $\theta \geq 90^\circ$, the strategy of Equation (\ref{eq:two-projection-gradient}) proves to be more effective than the strategy of Equation (\ref{eq:two-projection-gradient-average}) in handling gradient deviations. This indicates that the second strategy can successfully mitigate gradient degradation while addressing conflicting gradients. Specifically, the gradient $\g$ can be considered an {equilibrium gradient} that strikes a suitable balance between reward optimization and safety constraints. This implies that an increase in the reward or cost expected gradient cannot be achieved by altering its gradient manipulation, given that other gradients remain unmodified.

% which means no one can increase itself expected gradient by changing itself strategy while the other gradients keep theirs unchanged.

% \textbf{When gradient is conflict:}
 \begin{figure}[tb!]
 \centering
 \subcaptionbox{}
%  {
% \includegraphics[width=0.65\linewidth, angle=0]{files/figures/method/PCRPO-projection-analysis.pdf}   
% }
 {
\includegraphics[width=0.65\linewidth, angle=0]{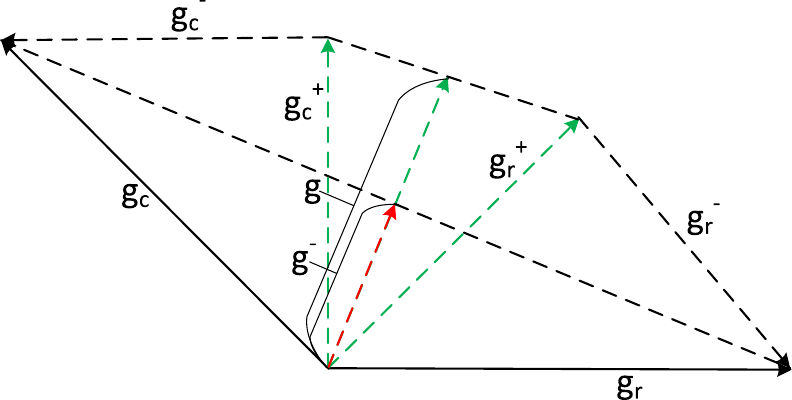}   
}
\subcaptionbox{}
%  {
% \includegraphics[width=0.23\linewidth, angle=0]{files/figures/method/PCRPO-average-full-analysis.pdf}   
% }
 {
\includegraphics[width=0.23\linewidth, angle=0]{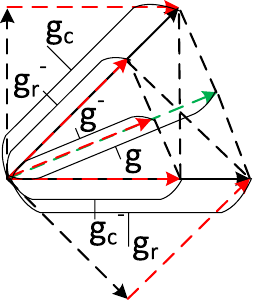}   
}
\subcaptionbox{}
%  {
% \includegraphics[width=0.65\linewidth, angle=0]{files/figures/method/PCRPO-original-projection.pdf}   
% }
 {
\includegraphics[width=0.65\linewidth, angle=0]{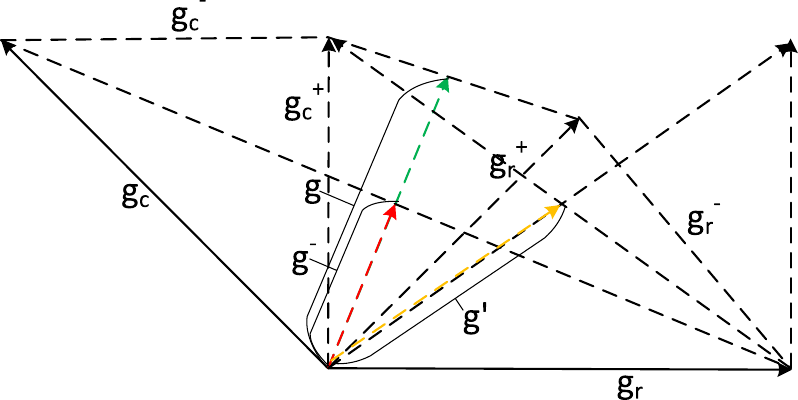}   
}
    % \vspace{-5pt}
 	\caption{\normalsize Analysis of Soft switching through gradient manipulation.
 	} 
  \label{fig:finish-pcrpo-projection-case-one-two-analysis}
  % \vspace{-15pt}
 \end{figure} 
% % \textbf{When gradient is not conflict}
%  \begin{figure}[htbp!]
%  \centering
%  % \subcaptionbox{}
%  {
% \includegraphics[width=0.35\linewidth, angle=0]{files/figures/method/PCRPO-average-full-analysis.pdf}   
% }
%     \vspace{-0pt}
%  	\caption{\normalsize Analysis for Case 2: Soft switching through gradient manipulation.
%  	} 
%   \label{fig:finish-pcrpo-projection-case-two-analysis}
%  \end{figure} 

 \begin{comment}
     % \textbf{When gradient is conflict:}
 \begin{figure}[htbp!]
 \centering
 % \subcaptionbox{}
 {
\includegraphics[width=0.6\linewidth, angle=0]{files/figures/method/PCRPO-projection-analysis.pdf}   
}
    \vspace{-0pt}
 	\caption{\normalsize Analysis for Case 1: Soft switching through gradient manipulation.
 	} 
  \label{fig:finish-pcrpo-projection-case-one-analysis}
 \end{figure} 
 \end{comment}
% \subsubsection{$\theta \leq 90^\circ$:}
 % With the condition of $\theta \leq 90^\circ, cos(\theta) \leq 0$, as shown in Figure~\ref{fig:finish-pcrpo-projection-case-two-analysis}, with the projection gradient of gradient $\g_r$ on the normal plane of gradient $\g_c$ and the projection gradient of gradient $\g_c$ on the normal plane of gradient $\g_r$, we can have gradients $\g_r^-$ and $\g_c^-$. we can observe,
 Under the circumstance where $\theta < 90^\circ$ and $cos(\theta) > 0$, as illustrated in Figure~\ref{fig:finish-pcrpo-projection-case-one-two-analysis} (b), the projection gradients of $\g_r$ on the normal plane of $\g_c$ and $\g_c$ on the normal plane of $\g_r$ yield gradients $\g_r^-$ and $\g_c^-$. Upon observation, it can be noted that,  
 % In this example, The gradient is handled through the second strategy, the Equation (\ref{eq:two-projection-gradient}), can be observed at Equation~(\ref{eq:theta-less-than-90-projection}), and the gradient addressed via the first strategy, Equation~(\ref{eq:two-projection-gradient-average}), is shown in Equation (\ref{eq:theta-less-than-90-average}). Upon observation, it can be noted that, under this condition, the first strategy is better than the second strategy for handling the deviations of reward and cost gradients, the first strategy can avoid the gradient degradation while reducing the deviation of gradients.
 % \vspace{-10pt}
 \begin{equation}
  \begin{aligned}
  \label{eq:theta-less-than-90-projection}
    &\g^- = \frac{{\g_{r}^+} + {\g_{c}^+}}{2}    
    \\
    = &\resizebox{!}{0.68cm}{$\frac{\left(\frac{\g_r}{\|\g_r\|} \|\g_r\| - \frac{cos(\theta)\|\g_r\| \|\g_{c}\|}{\|\g_{c}\|^2} \g_{c}\right) + \left(\frac{\g_c}{\|\g_c\|}\|\g_{c}\| - \frac{cos(\theta)\|\g_c\| \|\g_{r}\|}{\|\g_{r}\|^2} \g_{r}\right)}{2}$},
\end{aligned}  
\end{equation}
with $\theta < 90^\circ, cos(\theta)  > 0$, the following property holds, 
 \begin{equation}
  \begin{aligned}
    -\left( \frac{cos(\theta)\|\g_r\| \|\g_{c}\|}{\|\g_{c}\|^2} \right) - \left( \frac{cos(\theta)\|\g_c\| \|\g_{r}\|}{\|\g_{r}\|^2} \right) < 0,
\end{aligned}  
\end{equation}
\begin{equation}
  \begin{aligned}
  \label{eq:theta-less-than-90-average}
    \g = \frac{{\g_{r}} + {\g_{c}}}{2}
    = \frac{\frac{\g_r}{\|\g_r\|}\|\g_{r}\| + \frac{\g_c}{\|\g_c\|}\|\g_{c}\| }{2}.
\end{aligned}  
\end{equation}
Thus, we can observe ${\|\g\|} > {\|\g^-\|}$.
% \begin{equation}
%   \begin{aligned}
%    {\|\g\|} > {\|\g^-\|}.
% \end{aligned}  
% \end{equation}
In this example, $\theta < 90^\circ$, the gradient managed by the second strategy, as illustrated in Equation (\ref{eq:two-projection-gradient}), can be observed in Equation(\ref{eq:theta-less-than-90-projection}). Concurrently, the gradient addressed using the first strategy, as depicted in Equation(\ref{eq:two-projection-gradient-average}), is presented in Equation (\ref{eq:theta-less-than-90-average}). Upon examination, it becomes evident that under these conditions, the first strategy surpasses the second strategy in effectively handling deviations in reward and cost gradients. Furthermore, the first strategy is capable of mitigating gradient degradation while simultaneously reducing gradient deviation.

The gradient projection manipulation used in this study is inspired by gradient manipulation \cite{yu2020gradient, chen2021generalized, liu2021conflict, zhou2022convergence}, and we further leverage it for reward and safety balance. Specifically, as illustrated in Figure~\ref{fig:finish-pcrpo-projection-case-one-two-analysis} (c), the yellow dashed line denotes the updated gradient, $\g'$, generated by Algorithm 1 of gradient surgery \cite{yu2020gradient}, the angle, $\theta^s$, between $\g'$ and the original cost gradient, $\g_c$, remains greater than $90^\circ$. This observation implies that the relation between $\g'$ and $\g_c$ continues to exhibit a conflicting nature, potentially resulting in inadequate handling of the original cost gradient. Based on the subsequent analysis as shown in Equation~(\ref{eq:original-projection-gradient}), our gradient manipulation approach demonstrates improved performance. Algorithm 1 of gradient surgery~\cite{yu2020gradient} does not consider cases where $\theta < 90^\circ$, which might be insufficient for addressing optimization oscillations effectively. Thus, on the basis of gradient manipulation \cite{yu2020gradient, chen2021generalized, liu2021conflict, zhou2022convergence}, we consider both conflicting and non-conflicting gradient scenarios in safe RL. In the upcoming experiment section, we also provide ablation experiments to investigate the effectiveness of gradient manipulation methods regarding reward and safety optimization.
\vspace{-3pt}
\begin{equation}
\begin{aligned}
\label{eq:original-projection-gradient}
    \g' &= \frac{{\g_{r}} + {\g_{c}^+}}{2}
    =\frac{\g_r + \left(\g_{c} - \frac{\g_c \cdot \g_{r}}{\|\g_{r}\|^2} \g_{r}\right)}{2},
    \\
    &= \frac{\frac{\g_r}{\|\g_r\|} \|\g_r\|+ \left(\frac{\g_c}{\|\g_c\|}\|\g_{c}\| - \frac{cos(\theta)\|\g_c\| \|\g_{r}\|}{\|\g_{r}\|^2} \g_{r}\right)}{2},\\ &
    \Longrightarrow \|\g\| > \|\g'\|. 
\end{aligned}   
\end{equation}

\subsection{A Framework for Safe Reinforcement Learning with Soft Switching}

In this section, we present a comprehensive framework referred to as PCRPO, which iteratively optimizes performance until convergence is achieved. As demonstrated in Algorithm 1 in Appendix~A, we propose a novel approach with slack techniques to address the deviation of reward and cost gradients, particularly for conflicting gradients. 

\textbf{Case one:} In the event that the slack value tends toward infinity, i.e., $h^+ \rightarrow +\infty$ and $h^- = 0$, the optimization process is adapted based on the satisfaction of safety constraints. When a safety violation occurs, the optimization exclusively focuses on safety by employing Equation (\ref{eq:update-policy-parameters-for-safety}), where $w$ is the parameters represented by neural networks, $\eta$ is the step size of gradient update. Conversely, if safety constraints are satisfied, the optimization process incorporates the projection gradient, as delineated in Equation (\ref{eq:projection-gradient-one}). \textbf{Case two:}
In the event that the slack value is denoted by $h^+ = 0$ and $h^- \rightarrow -\infty$, a safety violation necessitates the enhancement of the reward and concurrent reduction of cost by employing Equation (\ref{eq:projection-gradient-one}). Conversely, when safety requirements are fulfilled, the focus shifts solely to the optimization of reward performance, as demonstrated by Equation (\ref{eq:update-policy-parameters-for-reward}).

\textbf{Case three:}
 In situations where the slack value is confined to the range of $+\infty > h^+ > 0$ and $0 > h^- > -\infty$,  several circumstances can be observed. If upper slack, lower slack, and safety violations occur simultaneously, the optimization process is devoted solely to addressing safety concerns, as indicated by Equation (\ref{eq:update-policy-parameters-for-safety}). In the absence of upper slack violations, while lower slack and safety violations transpire concurrently, the strategy involves enhancing the reward and concurrently reducing the cost by employing Equation (\ref{eq:projection-gradient-one}). Conversely, when upper slack and safety violations are not present, but a lower slack violation persists, the same approach of augmenting the reward and minimizing the cost is implemented using Equation (\ref{eq:projection-gradient-one}). Finally, in the absence of violations related to upper slack, lower slack, and safety, the primary focus is directed towards optimizing reward performance, as demonstrated by Equation (\ref{eq:update-policy-parameters-for-reward}).
 \begin{align} 
\label{eq:update-policy-parameters-for-reward}
w_{t+1}=w_t+\eta \bar{\Delta}_t^r, \g^r = \bar{\Delta}_t^{r},
\end{align}
% \vspace{-10pt}
\begin{align} 
\label{eq:update-policy-parameters-for-safety}
w_{t+1}=w_t-\eta \bar{\Delta}_t^{c_i}, \g^c =  -\bar{\Delta}_t^{c_i},
\end{align}
where from Lemma 5.1 of \cite{agarwal2021theory}, we have 
\begin{equation}
\bar{\Delta}_t^{r} = (1-\gamma)^{-1} \bar{Q}_t^{r_{t}}(s, a), \bar{\Delta}_t^{c_i} = (1-\gamma)^{-1} \bar{Q}_t^{c_{i, t}}(s, a).
\end{equation}
\begin{equation}
\label{eq:projection-gradient-one}
w_{t+1}=w_t + \eta \cdot\left(\frac{\g_r^++\g_c^+}{2}\right), w_{t+1}=w_t + \eta \cdot\left(\frac{\g_r+\g_c}{2}\right).
\end{equation}

% To achieve the above requirements,
Inspired by CRPO~\cite{xu2021crpo}, we implement our algorithm within the context of the primal optimization setting. Similarly, we initially evaluate the policy and subsequently improve it while addressing safety constraints.
 
\subsubsection{Policy Evaluation}
During the policy evaluation step, the objective is to learn Q-functions that accurately evaluate the previous policy $\pi_t$. To accomplish this, we train distinct Q-functions for both reward and constraints.
% \paragraph{Temporal difference (TD) learning.}
% Within TD learning, each iteration is expressed as
\begin{align} \label{eq: TD learning}
    \resizebox{!}{0.315cm}{$Q^{\pi_w}_{i,k+1}(s,a) = Q^{\pi_w}_{i,k} + \ell_k \left[ r_i(s,a) + \gamma Q^{\pi_w}_{i,k}(s^\prime,a^\prime)  - Q^{\pi_w}_{i,k}(s,a) \right]$},  
\end{align}
where $s \sim \mu_{\pi_w}, a\sim  \pi_w(s), s^\prime \sim \mathrm{P}(\cdot\mid s,a), a^\prime \sim \pi_w(s^\prime)$, $\ell_k$ is the learning rate and $i$ denotes the reward or any of the constraints.  $\bar{Q}_{i}(s,a)$ can be estimated via $Q^{\pi_w}_{i,K_{\text{TD}}}(s,a)$, where $K_{TD}$ is the iteration number of using TD learning methods. 

\subsubsection{Policy Improvement for Reward and Safety}
% \subsubsection{Projected Natural Policy Gradient}
The policy gradient \cite{sutton1999policy} of the reward value function $f_r\left(\pi_w\right)$ has been derived as $\nabla f_r\left(\pi_w\right)=\mathbb{E}\left[Q^{\pi_w}_r(s, a) \phi_w(s, a)\right]$, where $\phi_w(s, a):=\nabla_w \log \pi_w(a \mid s)$ is the score function. Similarly, for the value function of cost $i$, we have $\nabla f_{c_i}\left(\pi_w\right)=\mathbb{E}\left[Q^{\pi_w}_{c_i}(s, a) \phi_w(s, a)\right]$.

In scenarios where the optimization of both reward and safety $i$ is desired, it is necessary to select a non-conflicting gradient descent $\boldsymbol{d}$ on the natural gradients of reward and cost gradients. This selection aims to optimize reward and safety individually, subject to the constraint that the KL divergence between the updated and previous policy remains below a specified threshold.
% \begin{equation}
\begin{align}
% \label{eq:determine-update-direction}
    \boldsymbol{d} &= \frac{{\g_{r}^+} + {\g_{c_i}^+}}{2} \ \
    \text{or} \ \ \frac{{\g_{r}} + {\g_{c_i}}}{2}. 
    \label{eq:determine-update-direction-equation}
\end{align}
% \end{equation}

\subsubsection{Correlation-Reduction for Stochastic Gradient Manipulation}

In practical applications, the challenge of acquiring imprecise policy gradient feedback is frequently encountered. This imprecision stems from the restricted number of sampled trajectories employed to estimate $Q^{\pi_{w_t}}_r$ or $Q^{\pi{w_t}}_{c_i}$, subsequently introducing stochastic noise into the system. The study conducted in \cite{zhou2022convergence} has revealed that, within a stochastic setting, conventional gradient manipulation techniques may fail to converge to a optimal solution.

Let the weight factors $\boldsymbol{\lambda}t=(\lambda_t^r, \lambda_t^{c_1}, \ldots, \lambda_t^{c_n})$ be represented as $\boldsymbol{d} = \lambda_t^r \g_r + \lambda_t^{c_i} \g_{c_i}$ for the time step $t$. The primary cause of this convergence failure lies in the substantial correlation between the weight factors $\boldsymbol{\lambda}t$ and the stochastic gradients, resulting in a biased composite gradient. To address this issue within the context of the PCRPO framework, we concentrate on the specific conditions that ensure the variance of the natural policy gradient estimator progressively approaches zero. One possible approach to achieve this involves utilizing TD learning in \eqref{eq: TD learning} to estimate $Q^{\pi{w_t}}i$, assuming $K_{\text{TD}}$ is adequately large.

\subsection{Comparison to CRPO}
Compared to CRPO \cite{xu2021crpo}, a primal safe RL algorithm, our algorithm exhibits two distinct differences: the addition of upper and lower slack values to the constraint thresholds, and the unique approach we take to optimize the policy concerning both reward and safety. CRPO focuses on optimizing for reward, only shifting to safety optimization if a safety constraint is hard violated. This can lead to a back-and-forth between reward and cost optimizations, particularly when constraints are near their threshold boundaries. To circumvent such oscillations and to prevent any performance degradation in other objectives, we employ a projected gradient descent approach. This ensures a balanced and efficient way of handling both reward and safety concerns.

% \begin{center}
% \textbf{\Large End to rewrite}
% \end{center}

\subsection{Convergence Analysis}
% Performance improvement bound for soft switching of gradient manipulation
For the iterates $\pi_{w_t}$ manipulated using our proposed method, \textbf{we can guarantee performance monotonic improvement and convergence to the optimal performance in cases where $180^\circ > \theta \geq 0^\circ$}. Our theorem enables the derivation of other settings in a straightforward manner. Please refer to Appendix~B for detailed theorems and their corresponding proofs. %\ref{appendix:monotonic-guarantee}

\section{Experiments}
\label{section:experiments}

In the experimental section, we investigate the constraint satisfaction of policies trained using our proposed method and compare its performance with state-of-the-art (SOTA) safe RL algorithms. Employing the  developed benchmark, \textit{Safety-MuJoCo Benchmark}, we first compare our approach with a representative primal optimization-based safe RL method, CRPO~\cite{xu2021crpo}, a strong baseline established in 2021. CRPO demonstrates superior performance in comparison to several SOTA baselines such as PDO~\cite{ray2019benchmarking}.

To further emphasize the effectiveness of our method, we implement our algorithm on a popular benchmark,  \textit{Omnisafe}~\cite{ji2023omnisafe}, and  compare it with several representative primal-dual optimization-based safe RL methods. These methods include SOTA baselines, such as PCPO~\cite{yang2020projection}, CUP~\cite{yang2022constrained}, and PPO Lagrangian (PPOLag)~\cite{ji2023omnisafe}. By contrasting our approach with these established methods, we aim to demonstrate the potential advantages and improvements that our proposed method offers in terms of safety and efficiency for RL applications with safety constraints.

In particular, our proposed method belongs to the category of primal optimization-based safe RL approaches, thereby avoiding the challenges associated with hyperparameter tuning related to dual variables. Additionally, our method does not necessitate feasible initialization, in contrast to some primal-dual optimization-based methods where poor initialization can adversely impact performance~\cite{xu2021crpo}. By circumventing these challenges, our proposed method aims to provide a more reliable and efficient solution for ensuring safety in RL applications. Additionally, we conduct ablation studies to investigate the impact of diverse cost limits and sensitivity to slack bounds.

\subsection{Experiments on \textit{Safety-MuJoCo Benchmark}}

We have developed a \textit{Safety-MuJoCo Benchmark}\footnote{\url{https://github.com/SafeRL-Lab/Safety-MuJoCo.git}} based on MuJoCo~\cite{todorov2012mujoco, towers_gymnasium_2023} to evaluate the performance of safe RL algorithms. This benchmark differs from traditional safe RL benchmarks, such as \textit{Omnisafe}\footnote{\url{https://github.com/PKU-Alignment/omnisafe.git}}~\cite{ji2023omnisafe} that is developed based on Safety Gym\footnote{\url{https://github.com/openai/safety-gym.git}}~\cite{ray2019benchmarking}. In \textit{Omnisafe}, the cost constraint is set as the velocity limit, and the reward is determined by the speed at which the robot runs. Conversely, our benchmark considers not only velocity constraints but also the health of the robot. For example, whether the robot falls or whether its joints exceed the limit values of motion control are taken into account. For a comprehensive description of these considerations, please refer to Appendix~C. As depicted in Figure~\ref{fig:compared-pcrpo-walker-reacher-humanoidstandup-cost-limit-40-slack-5} (a) and (b), our method demonstrates remarkably superior performance in comparison to CRPO with respect to both reward maximization and safety preservation. Similarly, Figure~\ref{fig:compared-pcrpo-walker-reacher-humanoidstandup-cost-limit-40-slack-5} (c) and (d) illustrate our method's significant improvement over CRPO in terms of reward and cost performance.

\begin{figure}[tb!]
 \centering
 \subcaptionbox{}
 {
\includegraphics[width=0.4725\linewidth]{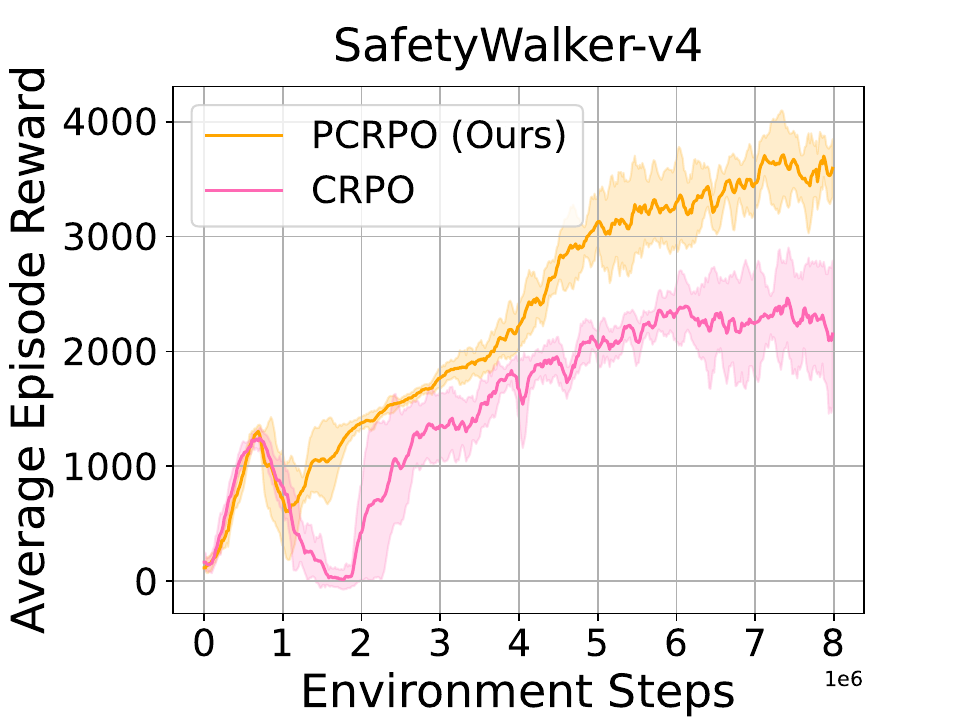}
}
 \subcaptionbox{}
 {
\includegraphics[width=0.4725\linewidth]{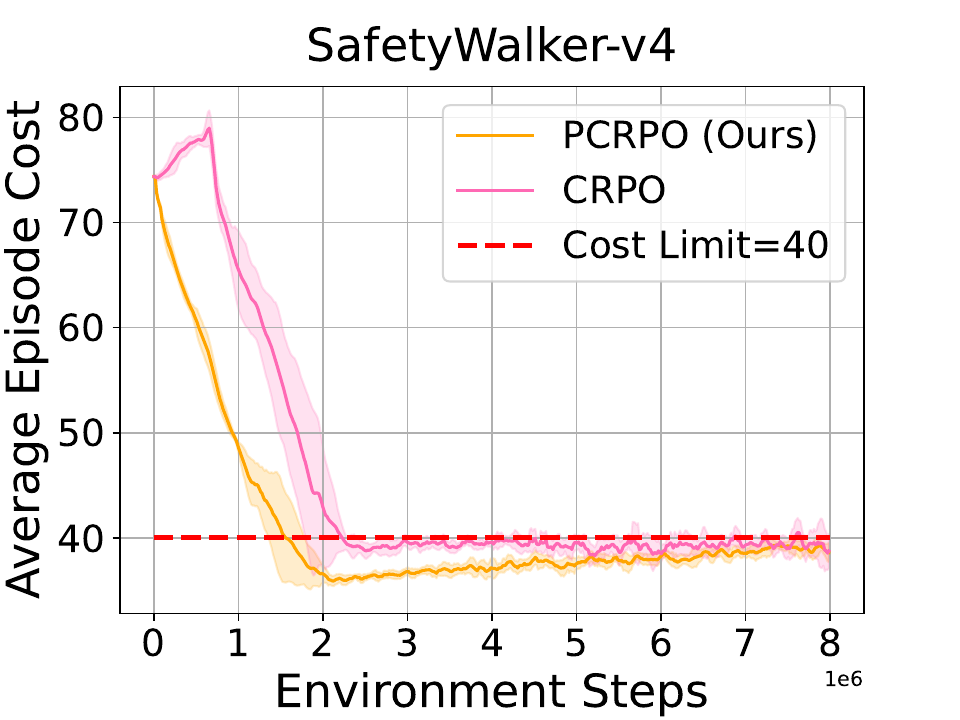}
}
\subcaptionbox{}
 {
\includegraphics[width=0.4725\linewidth]{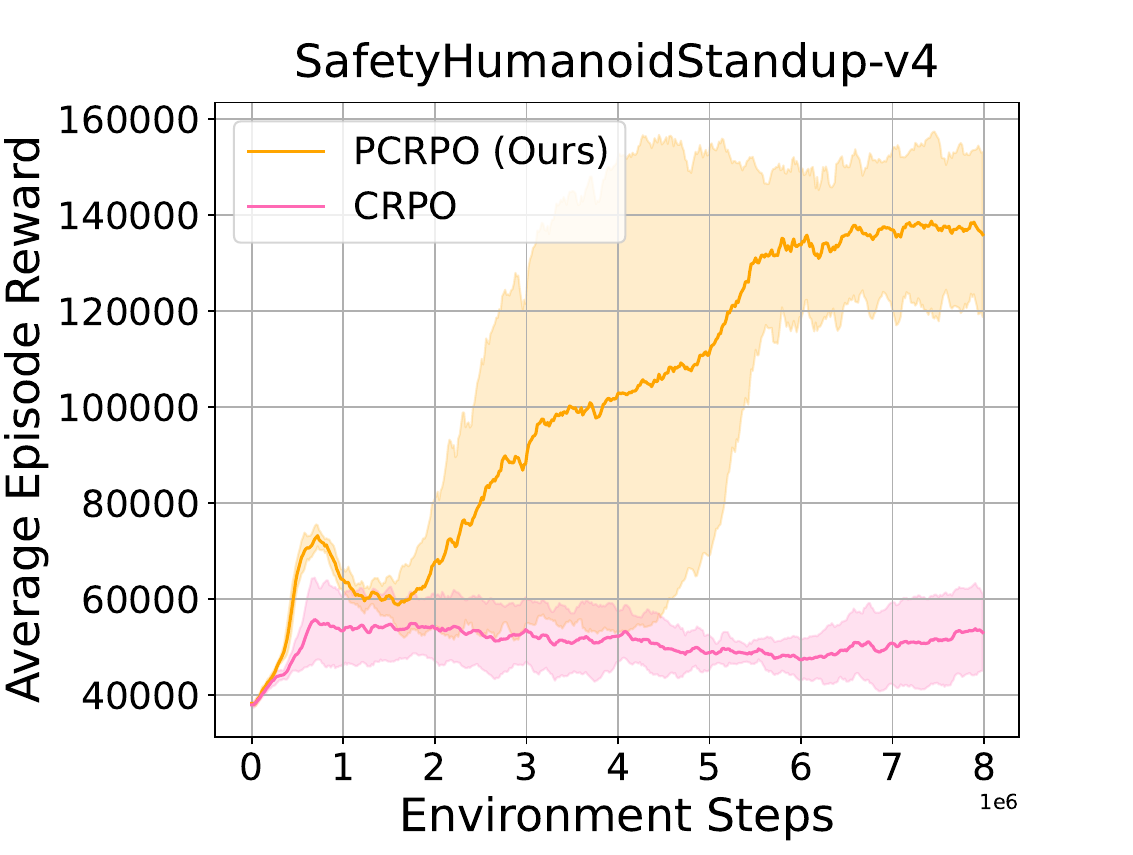}
}
 \subcaptionbox{}
 {
\includegraphics[width=0.4725\linewidth]{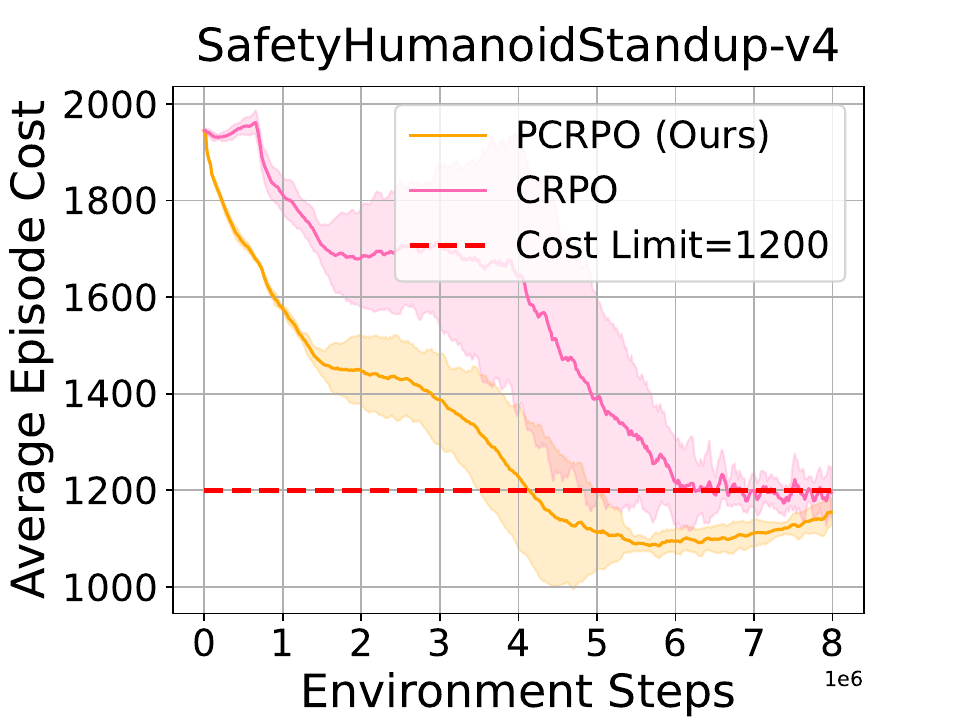}
}

 	\caption{\normalsize   Compared with CRPO on the SafetyWalker and SafetyHumanoidStandup Tasks. To encourage more learning exploration, we initiate the optimization of safety after $640000$ steps. 
 	} 
  \label{fig:compared-pcrpo-walker-reacher-humanoidstandup-cost-limit-40-slack-5}
  % \vspace{-15pt}
 \end{figure}

\subsection{Experiments on \textit{Omnisafe} Benchmark}

We implement our algorithm on \textit{Omnisafe} and compare its performance with several SOTA baselines within the \textit{Omnisafe} framework. The safety constraint is set as a constant threshold, where if the agent moves at a higher velocity than this threshold, it incurs a cost of $1$ per time step. We test our method, PCRPO, alongside PCPO, CUP, and PPOLag methods on various environments, including Hopper and Ant. 

As illustrated in Figure~\ref{fig:compared-omnisafe-Ant-hopper-swimmer-baselines} (a) and (b), on the SafetyHopperVelocity-v1 task, our algorithm exhibits superior reward performance compared to SOTA baselines and maintains reliable safety. In contrast, SOTA baselines such as CUP and PPOLag struggle to ensure safety, and their reward performance is worse than our algorithm. Notably, our approach outperforms PCPO in both reward and safety performance. In Figure~\ref{fig:compared-omnisafe-Ant-hopper-swimmer-baselines} (c) and (d), our algorithm effectively ensures complete safety on the SafetyAntVelocity-v1 task while achieving comparable reward performance. Specifically, our algorithm demonstrates greater safety than CUP and PPOLag, which can not ensure safety on the task. While PCPO can also ensure safety, its reward performance is inferior to our algorithm. Furthermore, our algorithm demonstrates faster convergence than the baselines.

 \begin{figure}[t]
 \centering
 \subcaptionbox{}
 {
\includegraphics[width=0.4725\linewidth]{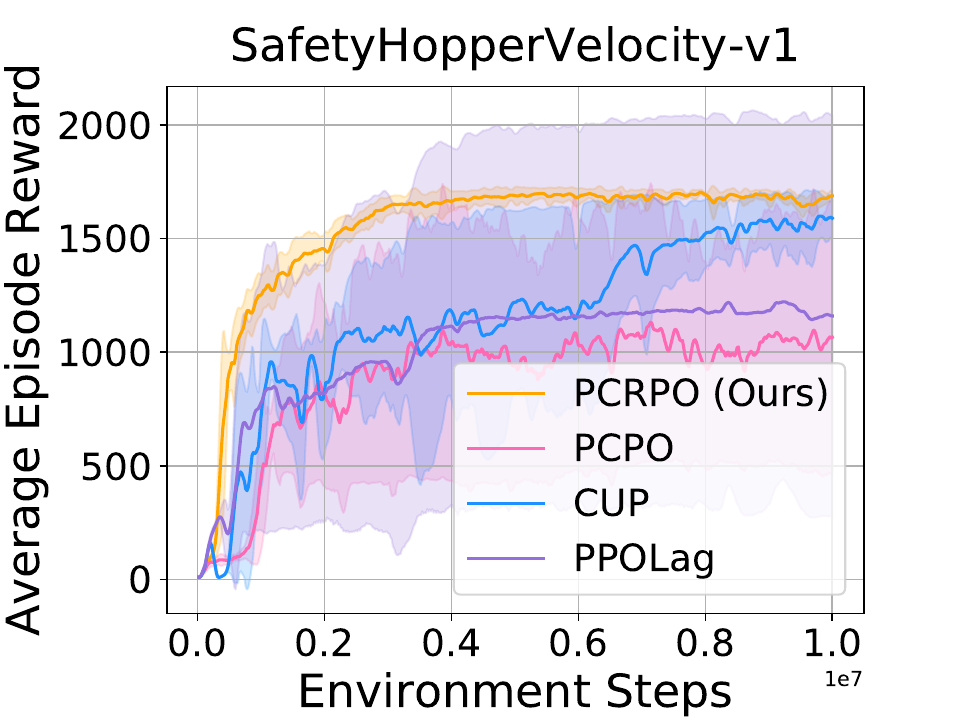}
}
 \subcaptionbox{}
 {
\includegraphics[width=0.4725\linewidth]{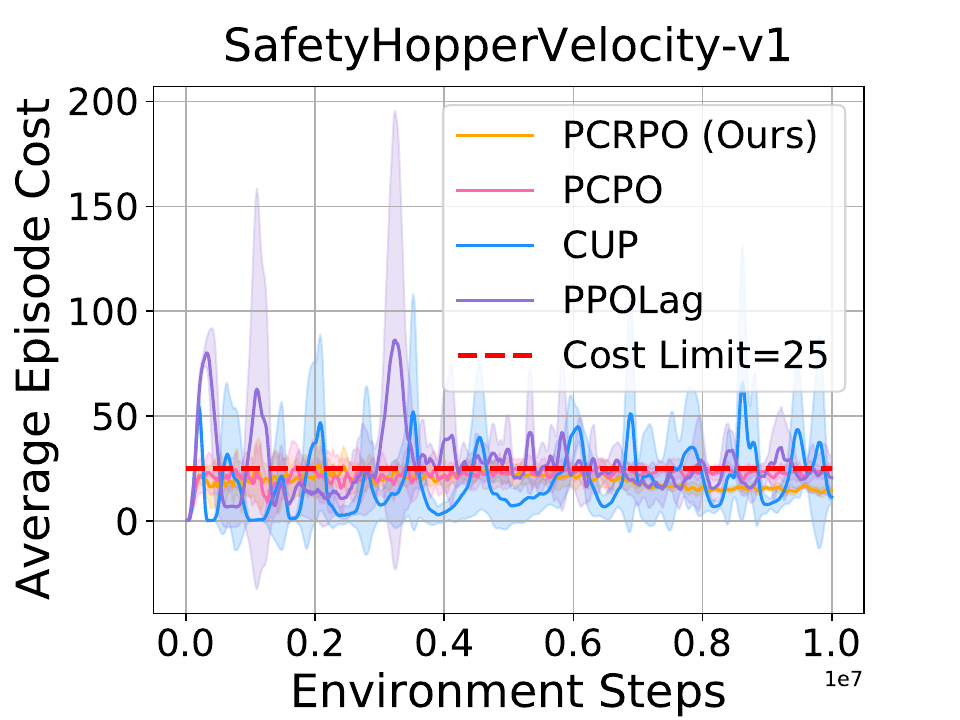}
}

 \subcaptionbox{}
 {
\includegraphics[width=0.4725\linewidth]{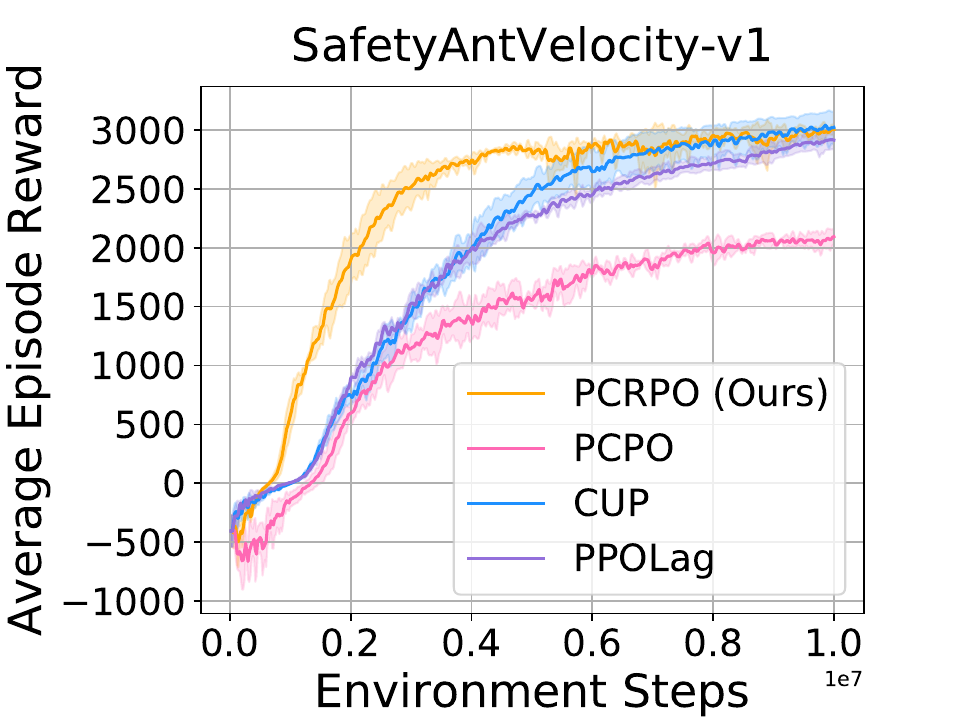}
}
 \subcaptionbox{}
 {
\includegraphics[width=0.4725\linewidth]{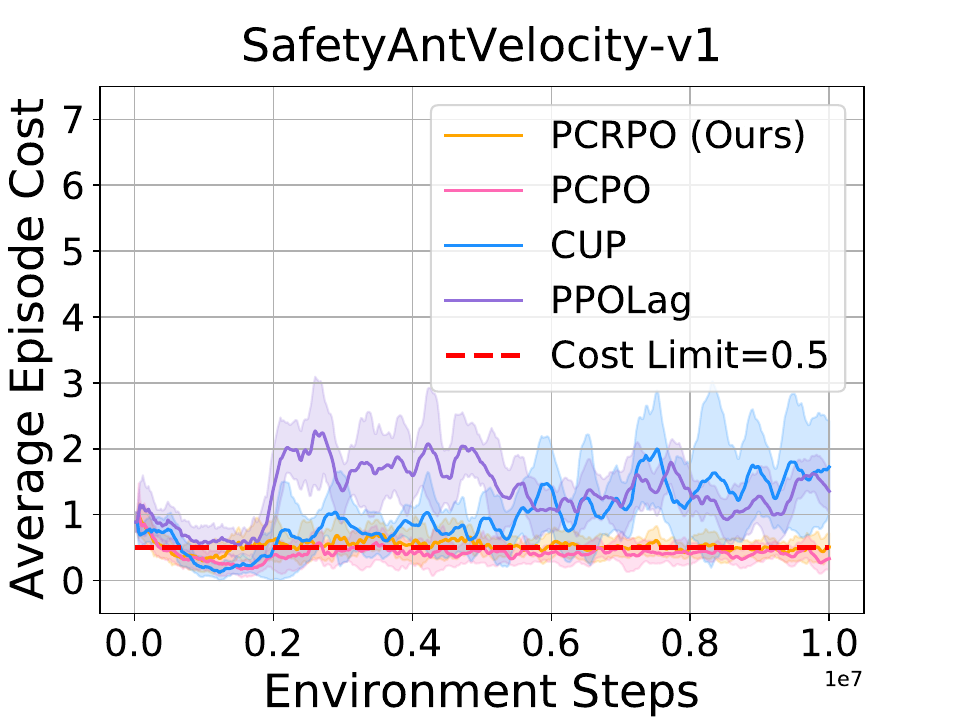}
}
% \vspace{-5pt}
 	\caption{\normalsize 
   Compared with PCPO, CUP, PPOLag baselines on SafetyHopperVelocity-v1 and SafetyAntVelocity-v1 tasks.  
 	} 
  \label{fig:compared-omnisafe-Ant-hopper-swimmer-baselines}
% \vspace{-15pt}
 \end{figure}

\subsection{Ablation Experiments}

  \begin{figure}[tb!]
 \centering

 \subcaptionbox{}
 {
\includegraphics[width=0.4725\linewidth]{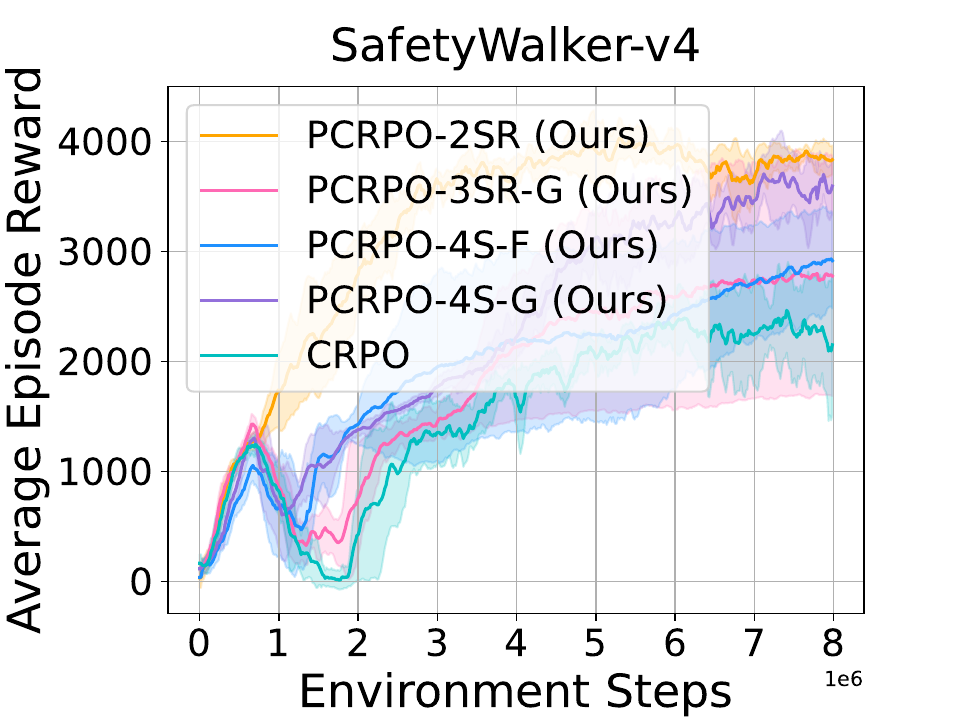}
}
 \subcaptionbox{}
 {
\includegraphics[width=0.4725\linewidth]{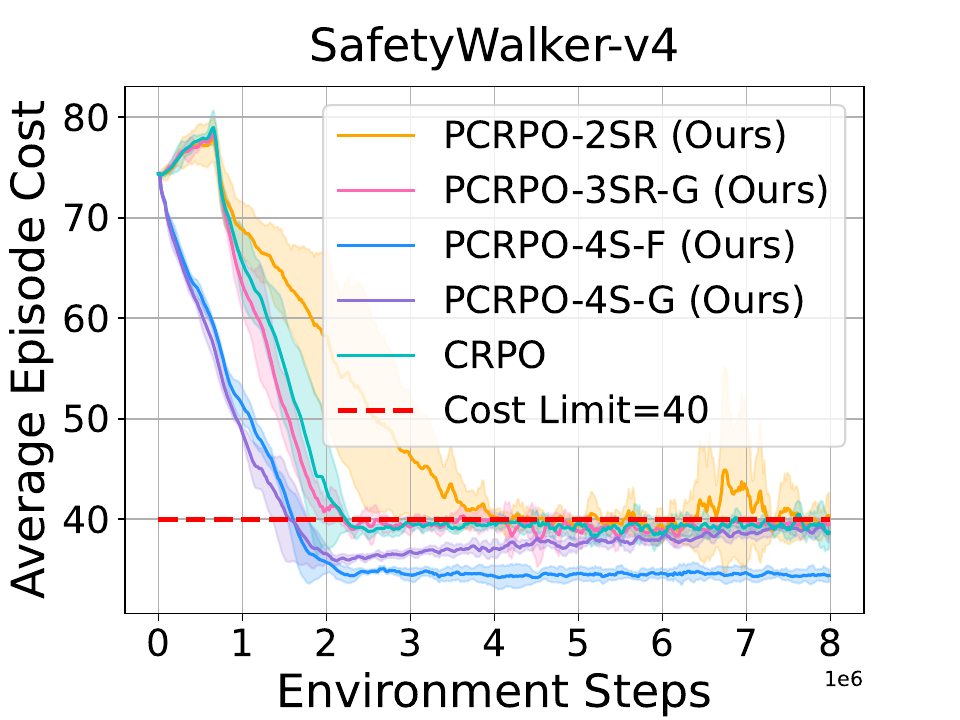}
}
 \subcaptionbox{}
 {
\includegraphics[width=0.4725\linewidth]{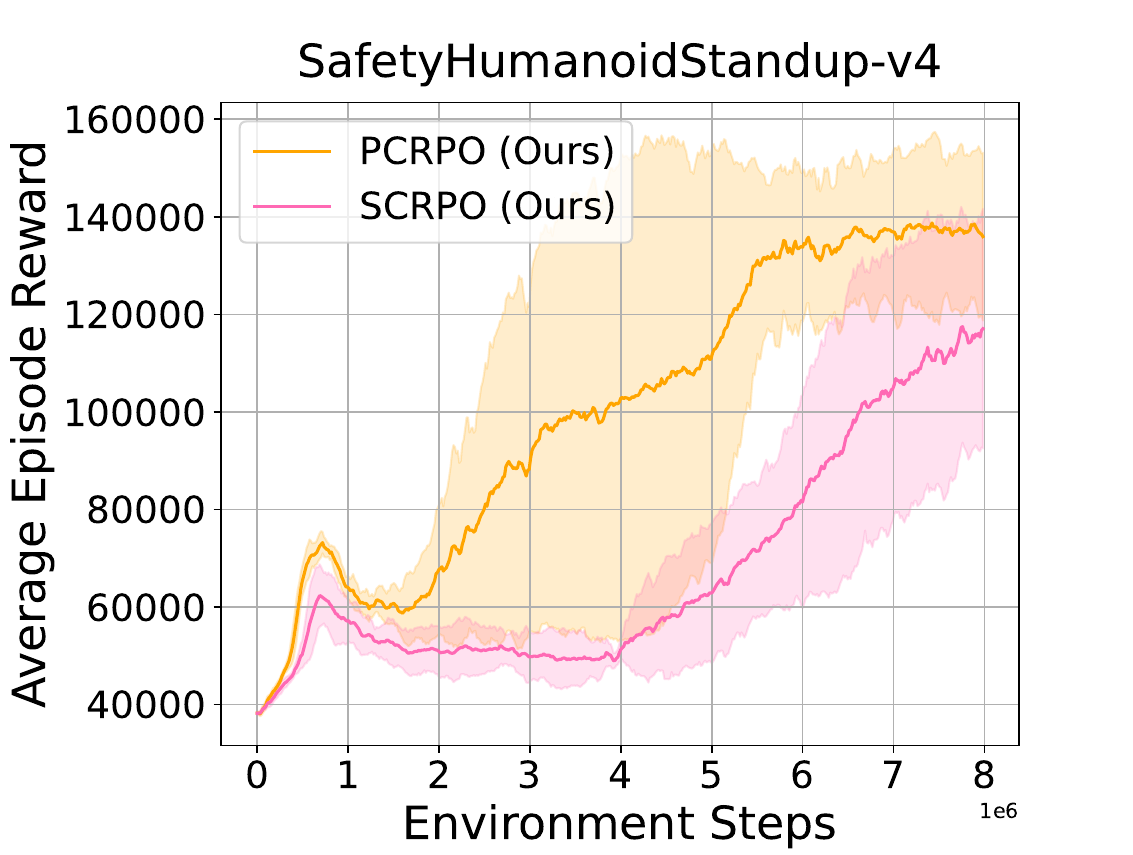}
}
 \subcaptionbox{}
 {
\includegraphics[width=0.4725\linewidth]{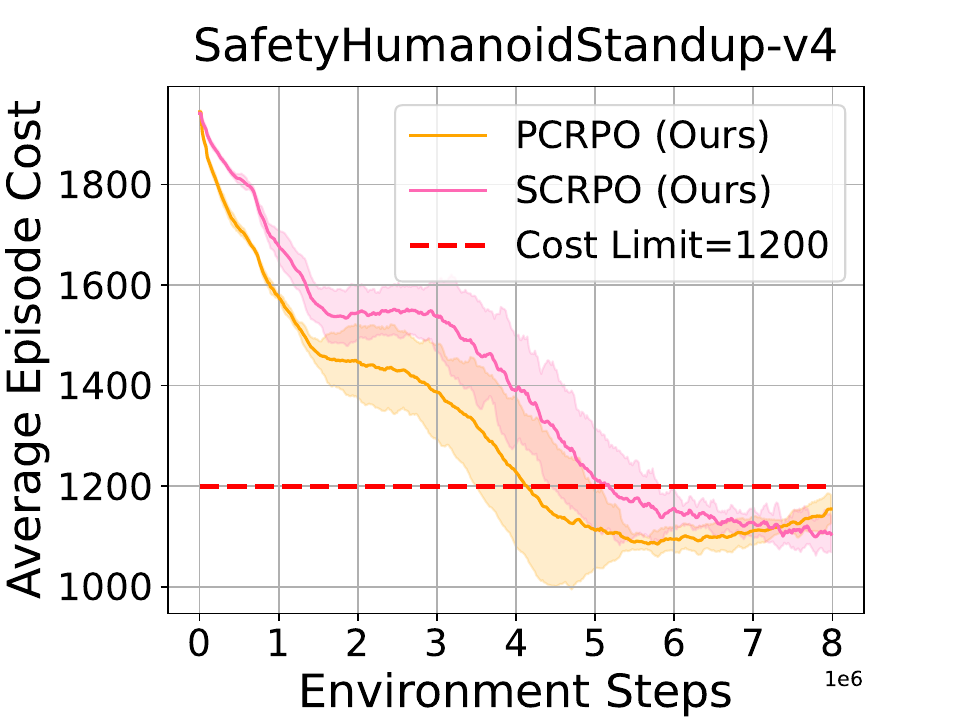}
}
    \vspace{-5pt}
 	\caption{\normalsize 
  ({a} $\&$ {b}) Ablation experiments of different slack settings on the SafetyWalker task.  ({c} $\&$ {d}) Ablation experiments of different gradient manipulation methods on the SafetyHumanoidStandup task. To encourage more learning exploration, we initiate the safety optimization after $640000$ steps on the SafetyWalker task. 
 	} 
  \label{fig:compared-omnisafe-Swimmer-ablation-baselines-cost-limit-all}
  \vspace{-15pt}
 \end{figure}

\subsubsection{Ablation Experiments of Slack Settings}

As illustrated in Figures~\ref{fig:compared-omnisafe-Swimmer-ablation-baselines-cost-limit-all} (a) and (b), we perform an ablation study on various slack settings. PCRPO-2SR represents $h^+_i \rightarrow +\infty, h^-_i =0$, where we primarily optimize reward while slightly ensuring safety. PCRPO-3SR-G denotes $h^+_i = 20, h^-_i =0$, with $h^+_i$ gradually decreasing to zero as the number of iteration steps increases. In this setting, we aim to optimize reward and safety simultaneously when $(b_i + h^+_i) > C_i > b_i$. 
PCRPO-4S-F corresponds to $h^+_i = 20, h^-_i =-20$, where we optimize safety and reward at the static slack boundary. As the experiment results show, our algorithm's cost value converges to the boundary. PCRPO-4S-G represents $h^+_i = 20, h^-_i =-20$, with $h^+_i$ and $h^-_i$ gradually decreasing to zero as the number of iteration steps increases. In this setting, we optimize safety and reward at the dynamic slack boundary, and as demonstrated by the experimental results, our algorithm's cost value converges to the cost limit while maintaining good reward performance. Notably, all our algorithms demonstrate superior results compared to CRPO in terms of balancing reward and safety optimization. This experimental setup allows us to analyze the impact of various slack configurations on the performance of our method, providing insights into the balance between reward and safety optimization.

\subsubsection{Ablation Experiments of Gradient Manipulation Methods}

As depicted in Figures~\ref{fig:compared-omnisafe-Swimmer-ablation-baselines-cost-limit-all} (c) and (d), we employ the gradient manipulation technique as outlined in Algorithm 1 of gradient surgery \cite{yu2020gradient} for learning safety, named as SCRPO. The experimental results demonstrate that PCRPO outperforms SCRPO in terms of safety and reward performance. These findings further corroborate the consistency of our theoretical analysis presented in the gradient analysis section.

\section{Conclusion}
\label{section:conclusion}

In this study, we address the issue of gradient conflicts between reward and cost by employing gradient manipulation. Specifically, we first propose a novel solution called PCRPO, which incorporates soft switching to balance reward and safety optimization in safe RL. Moreover, a slack technique is developed to help alleviate the conflict between reward and safety optimization. Our theoretical analysis demonstrates that our method can guarantee performance monotonic improvement while also analyzing the upper and lower bounds of the performance update. Then, we evaluate the effectiveness of our method using the \textit{Safety-MuJoCo Benchmark} that we developed, as well as a popular safe RL benchmark, \textit{Omnisafe}. Finally, the experimental results show that our method outperforms the strong baselines, indicating its superior performance in addressing the challenges associated with safe RL.

% \bigskip
% \noindent Thank you for reading these instructions carefully. We look forward to receiving your electronic files!

% \bibliography{aaai24}
\bibliography{files/ref}

\appendix
\onecolumn
\begin{center}
\textbf{\Large Appendix}
\end{center}
% \clearpage
\section{Practical Algorithm}
\label{appendix:practical-algorithms}

\begin{algorithm}[ht!] 
\caption{\textbf{PCRPO}: A Framework of Projection Constraint-Soft-Rectified Policy Optimization }
\label{alg:pcrpo-framework}
\begin{algorithmic}[1]
\STATE \textbf{Inputs}: initial policy with parameters $\pi_{w_0}$, positive slack value $h^+_t$, negative slack value $h^-_t$, the cost value of constraint $i$ as $V^{\pi_{w_0}}_{c_{i,t}}(\rho)$ at step $t$, the cost limit of constraint $i$ as $b_i$.
\FOR{$t = 0, \dots,T-1$}
\STATE Policy evaluation under $\pi_{w_t}$ involves estimating the values of rewards and constraints. %Policy evaluation under $\pi_{t}$ for all rewards and constraints.
\STATE Sample pairs $(s_j, a_j)$ from the buffer $\mathcal{B}_t$ according to the distribution $\rho \cdot \pi_{w_t}$ and compute the estimation $V^{\pi_{w_t}}_{r,t}(\rho)$ and $V^{\pi_{w_t}}_{c_{i,t}}(\rho)$, where $s_j$ represents the state and $a_j$ represents the action, $j$ is is the index for the sampled pairs.
\IF{ $h^+_t = +\infty, h^-_t = 0 $ }
\IF{For any $i$, $b_i > V^{\pi_{w_t}}_{c_{i,t}}(\rho)$}
\STATE Update policy ${\pi_{w_t}}$ to maximize reward value $V^{\pi_{w_t}}_{r,t}(\rho)$ with Equation~(\ref{eq:update-policy-parameters-for-reward}). 
\ELSE
\STATE Compute the update direction $\boldsymbol{d}$ with Equation~(\ref{eq:determine-update-direction-equation}), and update the projection policy leveraging the obtained $\boldsymbol{d}$ and Equation (\ref{eq:projection-gradient-one}).
\ENDIF
\ELSIF{$h^+_t = 0, h^-_t = -\infty $}
\IF{For any $i$, $b_i > V^{\pi_{w_t}}_{c_{i,t}}(\rho)$}
\STATE Compute the update direction $\boldsymbol{d}$ with Equation~(\ref{eq:determine-update-direction-equation}), and update the projection policy leveraging the obtained $\boldsymbol{d}$ and Equation (\ref{eq:projection-gradient-one}). 
\ELSE
\STATE Update policy to minimize cost $V^{\pi_{w_t}}_{c_{i,t}}(\rho)$ and ensure safety with Equation~(\ref{eq:update-policy-parameters-for-safety}). 
\ENDIF
\ELSIF{$+\infty> h^+_t > 0, 0 > h^-_t > -\infty $}
\IF{$h^+$ adaptive decreases}
\STATE $h^+_t \leftarrow h^+_t - h^+_t/T$ 
\ENDIF
\IF{$h^-_t$ adaptive increases}
\STATE $h^-_t \leftarrow h^-_t -  h^-_t/T$ 
\ENDIF
\IF{For any $i$, $V^{\pi_{w_t}}_{c_{i,t}}(\rho) > (h^+_t + b_i)$ }
\STATE  Choose any unsatisfied constraint $i_t$ and update policy to ensure safety with Equation~(\ref{eq:update-policy-parameters-for-safety}). 
\ELSIF{For any $i$, $(h^-_t + b_i)< V^{\pi_{w_t}}_{c_{i,t}}(\rho) < (h^+_t + b_i)$}
\STATE Compute the update direction $\boldsymbol{d}$ with Equation~(\ref{eq:determine-update-direction-equation}), and update the projection policy leveraging the obtained $\boldsymbol{d}$ and Equation (\ref{eq:projection-gradient-one}).
\ELSIF{For any $i$, $V^{\pi_{w_t}}_{c_{i,t}}(\rho) < (h^-_t + b_i)$}
\STATE Update policy ${\pi_{w_t}}_{t}$ to maximize reward $V^{\pi_{w_t}}_{r,t}(\rho)$ with Equation~(\ref{eq:update-policy-parameters-for-reward}). 
\ENDIF
\ENDIF
\ENDFOR{}
\STATE \textbf{Outputs}: ${\pi_{w_t}}_{\text{out}}$.
\end{algorithmic}
\end{algorithm}

% \clearpage
\section{Guarantees of  Monotonic Performance Improvement and  Convergence}
\label{appendix:monotonic-guarantee}

% \begin{comment}
    
% \end{comment}
% \textbf{For  performance, not just the reward performance:}

% [Performance improvement bound for soft switching of gradient manipulation]
\begin{theorem}\label{appendix-theorem: Performance improvement bound-case-one}
Under the assumption of Lipschitz continuity with a constant L, for the iterates $\pi_{w_t}$ generated through our gradient manipulation, when $180^\circ > \theta \geq 90^\circ$, as depicted in Equation(\ref{appendix-eq:upper-bound-lower-bound-theta-more-than-90-degree}), both the upper and lower bounds of performance updates can be observed. As demonstrated in Equation~(\ref{appendix-eq:lower-bound-positive-theta-more-than-90-degree}), {our method can guarantee  performance monotonic improvement and converge to the optimal performance.}

 \begin{equation}
    \begin{aligned}
    \label{appendix-eq:upper-bound-lower-bound-theta-more-than-90-degree}
       &\frac{1}{L} \cdot\left(\frac{5\|\g_r\|^2 + 5\|\g_c\|^2 -5cos(\theta)\|\g_r\|^2 -5cos(\theta)\|\g_c\|^2  + cos^2(\theta) - cos(\theta)\|\g_r\|\|\g_c\|}{8}\right)  \geq \\& f(w_{t+1}) - f(w_t)  \geq \eta \cdot\left(\frac{3\|\g_r\|^2 + 3\|\g_c\|^2 -3cos^2(\theta)\|\g_r\|^2 -3cos^2(\theta)\|\g_c\|^2 -  2cos^3(\theta)\|\g_r\|\|\g_c\| + 2cos(\theta)\|\g_r\|\|\g_c\|}{8}\right). 
    \end{aligned}
\end{equation}
In the case of $180^\circ > \theta \geq 90^\circ$, the following property holds,
 \begin{equation}
    \begin{aligned}  
\label{appendix-eq:lower-bound-positive-theta-more-than-90-degree}
       & {3\|\g_r\|^2 + 3\|\g_c\|^2  -  2cos^3(\theta)\|\g_r\|\|\g_c\| } > 3cos^2(\theta)\|\g_r\|^2 
 + 3cos^2(\theta)\|\g_c\|^2 -2cos(\theta)\|\g_r\|\|\g_c\|
       \\
       & \Longrightarrow f(w_{t+1}) - f(w_t) > 0.       
    \end{aligned}
\end{equation}

\end{theorem}

\begin{proof}

% In CRPO~\cite{xu2021crpo}, parameter update is as fellows:

% \textbf{For the reward performance:}

% \textbf{Note}: 

When $\theta \geq 90^\circ$, under the assumption that the gradient function, $\nabla f(w)$, exhibits Lipschitz continuity with a constant $L$, it can be deduced that the difference between the Hessian matrix $\nabla^2 f(w)$ and the scaled identity matrix $LI$ is a negative semi-definite matrix~\cite{zhou2018fenchel}. By leveraging this property, one can conduct a quadratic expansion of the function $f(w)$ in the vicinity of $f(w)$, which subsequently leads to the derivation of the following inequality:

 \begin{equation}
\begin{aligned}
& f\left(w_{t+1}\right) \geq f(w_t) + \nabla f(w_t)^T\left(w_{t+1}-w_t\right) -\frac{1}{2} L\left\|w_{t+1}-w_t\right\|^2, \\
&
f\left(w_{t+1}\right) \leq f(w_t) + \nabla f(w_t)^T\left(w_{t+1}-w_t\right)  +\frac{1}{2} L\left\|w_{t+1}-w_t\right\|^2.
\end{aligned}
\end{equation}  

% \begin{equation}
% \cos \left(\theta\right)=\frac{\mathbf{g}_{\mathbf{r}} \cdot \mathbf{g}_{\mathbf{c}}}{\left\|\mathbf{g}_{\mathbf{r}}\right\|\left\|\mathbf{g}_{\mathbf{c}}\right\|}.
% \end{equation}

With Lipschitz continuity, the assumption holds,
\begin{equation}
\eta \leq \frac{1}{L}.
\end{equation}

% \begin{align*}
%     \g_r^+ = \g_r - \frac{\g_r \cdot \g_{c}}{\|\g_{c}\|^2} \g_{c},
% \end{align*}

% \begin{align*}
%     \g_c^+ = \g_{c} - \frac{\g_c \cdot \g_{r}}{\|\g_{r}\|^2} \g_{r},
% \end{align*}

Leveraging gradient update functions~(\ref{eq:update-policy-parameters-for-reward}) and (\ref{eq:update-policy-parameters-for-safety}). For the case, $\theta \geq 90^\circ$, with Equations~(\ref{eq:initial-gradient-projection}), (\ref{eq:two-projection-gradient}) and (\ref{eq:cos-vector-gradient}), we have the performance improvement upper bound:

 \begin{equation}
\begin{aligned}
f\left(w_{t+1}\right)  \leq& f(w_t)+\nabla f(w_t)^T\left(w_{t+1}-w_t\right)+\frac{1}{2} L\left\|w_{t+1}-w_t\right\|^2 
\\
 =& f(w_t)+ (\g_r + \g_c)^T\left(\eta \cdot\left(\frac{\g_r^++\g_c^+}{2}\right)\right)+\frac{1}{2} L\left\|\eta \cdot\left(\frac{\g_r^++\g_c^+}{2}\right)\right\|^2 
\\
 =& f(w_t)+ (\g_r + \g_c)^T\eta \cdot\left(\frac{\left(\g_r - \frac{\g_r \cdot \g_{c}}{\|\g_{c}\|^2} \g_{c}\right) +\left(\g_{c} - \frac{\g_c \cdot \g_{r}}{\|\g_{r}\|^2} \g_{r}\right)}{2}\right)
\\
& +\frac{1}{2} L\left\|\eta \cdot\left(\frac{\left(\g_r - \frac{\g_r \cdot \g_{c}}{\|\g_{c}\|^2} \g_{c}\right) + \left(\g_{c} - \frac{\g_c \cdot \g_{r}}{\|\g_{r}\|^2} \g_{r}\right)}{2}\right)\right\|^2 
\\
 = & f(w_t)+ \eta \cdot\left(\frac{\|\g_r\|^2 + \|\g_c\|^2 -cos^2(\theta)\|\g_r\|^2 -cos^2(\theta)\|\g_c\|^2}{2}\right)
\\
& +\frac{1}{2} L \eta^2 \cdot \left(\frac{\|\g_r\|^2 + \|\g_c\|^2 -cos^2(\theta)\|\g_r\|^2 -cos^2(\theta)\|\g_c\|^2 + 2cos^3(\theta)\|\g_r\|\|\g_c\| - 2cos(\theta)\|\g_r\|\|\g_c\|}{4}\right)
\\
 \leq &  f(w_t)+ \frac{1}{L} \cdot\left(\frac{5\|\g_r\|^2 + 5\|\g_c\|^2 -5cos^2(\theta)\|\g_r\|^2 -5cos^2(\theta)\|\g_c\|^2 +  2cos^3(\theta)\|\g_r\|\|\g_c\| - 2cos(\theta)\|\g_r\|\|\g_c\|}{8}\right)
 \\  \Longrightarrow 
 f(w_{t+1}) &- f(w_t)  \leq \frac{1}{L} \cdot\left(\frac{5\|\g_r\|^2 + 5\|\g_c\|^2 -5cos^2(\theta)\|\g_r\|^2 -5cos^2(\theta)\|\g_c\|^2 +  2cos^3(\theta)\|\g_r\|\|\g_c\| - 2cos(\theta)\|\g_r\|\|\g_c\|}{8}\right) \geq 0.
\end{aligned}
\end{equation}  

Similarly, we can have the lower bound, 

\begin{equation}
    \begin{aligned}
       f(w_{t+1}) - f(w_t)  \geq \eta \cdot\left(\frac{3\|\g_r\|^2 + 3\|\g_c\|^2 -3cos^2(\theta)\|\g_r\|^2 -3cos^2(\theta)\|\g_c\|^2 -  2cos^3(\theta)\|\g_r\|\|\g_c\| + 2cos(\theta)\|\g_r\|\|\g_c\|}{8}\right). 
    \end{aligned}
\end{equation}

Thus,  performance updates at each iteration $t$ can be bounded as follows,
\begin{equation}
    \begin{aligned}
       &\frac{1}{L} \cdot\left(\frac{5\|\g_r\|^2 + 5\|\g_c\|^2 -5cos(\theta)\|\g_r\|^2 -5cos(\theta)\|\g_c\|^2  + cos^2(\theta) - cos(\theta)\|\g_r\|\|\g_c\|}{8}\right)  \geq \\& f(w_{t+1}) - f(w_t)  \geq \eta \cdot\left(\frac{3\|\g_r\|^2 + 3\|\g_c\|^2 -3cos^2(\theta)\|\g_r\|^2 -3cos^2(\theta)\|\g_c\|^2 -  2cos^3(\theta)\|\g_r\|\|\g_c\| + 2cos(\theta)\|\g_r\|\|\g_c\|}{8}\right). 
    \end{aligned}
\end{equation}

When $90^\circ \leq \theta < 180^\circ$, the following property holds,
       \\
\begin{equation}
    \begin{aligned}    
       & {3(1-cos^2)\|\g_r\|^2 + 3(1-cos^2)\|\g_c\|^2   > 2cos(\theta)(cos^2(\theta)-1)\|\g_r\|\|\g_c\| }
       \\       
       & {3sin^2(\theta)\|\g_r\|^2 + 3sin^2(\theta)\|\g_c\|^2   > -2cos(\theta)sin^2(\theta)\|\g_r\|\|\g_c\| }
       \\
       & {3\|\g_r\|^2 + 3\|\g_c\|^2   > -2cos(\theta)\|\g_r\|\|\g_c\| }, \text{ if } sin(\theta)\neq 0
       \\
       & { \frac{\|\g_r\|}{\|\g_c\|} + \frac{\|\g_c\|}{\|\g_r\|}  > 1> -\frac{2}{3}cos(\theta)},
       \\
       &  \Longrightarrow {3\|\g_r\|^2 + 3\|\g_c\|^2 -3cos^2(\theta)\|\g_r\|^2 -3cos^2(\theta)\|\g_c\|^2 -  2cos^3(\theta)\|\g_r\|\|\g_c\| + 2cos(\theta)\|\g_r\|\|\g_c\|} > 0, 
       \\
       & \Longrightarrow f(w_{t+1}) - f(w_t) > 0.
    \end{aligned}
\end{equation}

This implies that, through our gradient manipulation, in instances where $180^\circ > \theta \geq 90^\circ$, the proposed method can guarantee monotonic  performance improvement. Furthermore, the method demonstrates the ability to converge towards optimal performance.
\end{proof}

% [Performance improvement bound for soft switching of gradient manipulation]
\begin{theorem}\label{appendix-theorem: Performance improvement bound-case-two}
Under the assumption of Lipschitz continuity with a constant L, for the iterates $\pi_{w_t}$ generated through our gradient manipulation, in the case of $90^\circ > \theta \geq 0^\circ$, as depicted in Equation(\ref{appendix-eq:upper-bound-lower-bound-theta-less-than-90-degree}), both the upper and lower bounds of performance updates can be observed. As demonstrated in Equation~(\ref{appendix-eq:lower-bound-positive-theta-less-than-90-degree}), {our method can guarantee  performance monotonic improvement and converge to the optimal performance.}

\begin{equation}
\begin{aligned}
\label{appendix-eq:upper-bound-lower-bound-theta-less-than-90-degree}
\frac{1}{L} \cdot\left(\frac{5\|\g_r\|^2+10cos(\theta)\|\g_r\|\|\g_c\|+5\|\g_c\|^2}{8}\right)  \geq
f\left(w_{t+1}\right)  -f(w_t)
 \geq
 \eta \cdot\left(\frac{3\|\g_r\|^2+6cos(\theta)\|\g_r\|\|\g_c\|+3\|\g_c\|^2}{8}\right)
\end{aligned}
\end{equation}

% In this case, $cos(\theta) \geq 0$, we can have
\begin{equation}
\begin{aligned}
\label{appendix-eq:lower-bound-positive-theta-less-than-90-degree}
&\left(\frac{3\|\g_r\|^2+6cos(\theta)\|\g_r\|\|\g_c\|+3\|\g_c\|^2}{8}\right) > 0
\\
& \Longrightarrow
f\left(w_{t+1}\right)  -f(w_t) >0,
\end{aligned}
\end{equation}
\end{theorem}

\begin{proof}
Similar to Theorem~\ref{appendix-theorem: Performance improvement bound-case-one}, the upper bound is observed,
 \begin{equation}
\begin{aligned}
f\left(w_{t+1}\right)  
 \leq& f(w_t)+\nabla f(w_t)^T\left(w_{t+1}-w_t\right)+\frac{1}{2} L\left\|w_{t+1}-w_t\right\|^2 
\\
 =& f(w_t)+ (\g_r + \g_c)^T\left(\eta \cdot\left(\frac{\g_r+\g_c}{2}\right)\right)+\frac{1}{2} L\left\|\eta \cdot\left(\frac{\g_r+\g_c}{2}\right)\right\|^2 
\\
 =& f(w_t)+ \eta \cdot\left(\frac{\|\g_r\|^2+2cos(\theta)\|\g_r\|\|\g_c\|+\|\g_c\|^2}{2}\right)+ L\eta^2 \cdot\left(\frac{\|\g_r\|^2+2cos(\theta)\|\g_r\|\|\g_c\|+\|\g_c\|^2}{8}\right) 
 \\
 \leq& f(w_t)+ \frac{1}{L} \cdot\left(\frac{\|\g_r\|^2+2cos(\theta)\|\g_r\|\|\g_c\|+\|\g_c\|^2}{2}\right)+ \frac{1}{L} \cdot\left(\frac{\|\g_r\|^2+2cos(\theta)\|\g_r\|\|\g_c\|+\|\g_c\|^2}{8}\right) 
 \\
 =& f(w_t)+\frac{1}{L} \cdot\left(\frac{5\|\g_r\|^2+10cos(\theta)\|\g_r\|\|\g_c\|+5\|\g_c\|^2}{8}\right).
\end{aligned}
\end{equation} 

Similarly, we can have the lower bound,
\begin{equation}
\begin{aligned}
f\left(w_{t+1}\right)  
 \geq& f(w_t)+\nabla f(w_t)^T\left(w_{t+1}-w_t\right)-\frac{1}{2} L\left\|w_{t+1}-w_t\right\|^2 
\\
 =& f(w_t)+ (\g_r + \g_c)^T\left(\eta \cdot\left(\frac{\g_r+\g_c}{2}\right)\right)-\frac{1}{2} L\left\|\eta \cdot\left(\frac{\g_r+\g_c}{2}\right)\right\|^2 
\\
 =& f(w_t)+ \eta \cdot\left(\frac{\|\g_r\|^2+2cos(\theta)\|\g_r\|\|\g_c\|+\|\g_c\|^2}{2}\right)- L\eta^2 \cdot\left(\frac{\|\g_r\|^2+2cos(\theta)\|\g_r\|\|\g_c\|+\|\g_c\|^2}{8}\right)
 \\
 \geq& f(w_t)+ \eta \cdot\left(\frac{\|\g_r\|^2+2cos(\theta)\|\g_r\|\|\g_c\|+\|\g_c\|^2}{2}\right)- \eta \cdot\left(\frac{\|\g_r\|^2+2cos(\theta)\|\g_r\|\|\g_c\|+\|\g_c\|^2}{8}\right) 
 \\
 =& f(w_t)+\eta \cdot\left(\frac{3\|\g_r\|^2+6cos(\theta)\|\g_r\|\|\g_c\|+3\|\g_c\|^2}{8}\right). 
\end{aligned}
\end{equation} 

Thus, in this case of $90^\circ > \theta \geq 0^\circ$, at each iteration $t$, we can have the performance update bound, 

\begin{equation}
\begin{aligned}
\frac{1}{L} \cdot\left(\frac{5\|\g_r\|^2+10cos(\theta)\|\g_r\|\|\g_c\|+5\|\g_c\|^2}{8}\right)  \geq
f\left(w_{t+1}\right)  -f(w_t)
 \geq
 \eta \cdot\left(\frac{3\|\g_r\|^2+6cos(\theta)\|\g_r\|\|\g_c\|+3\|\g_c\|^2}{8}\right).
\end{aligned}
\end{equation}

In this case, $cos(\theta) \geq 0$, we can have
\begin{equation}
\begin{aligned}
&\left(\frac{3\|\g_r\|^2+6cos(\theta)\|\g_r\|\|\g_c\|+3\|\g_c\|^2}{8}\right) > 0
\\
& \Longrightarrow
f\left(w_{t+1}\right)  -f(w_t) >0,
\end{aligned}
\end{equation}
% \textbf{which means our method can guarantee the performance monotonic improvement, and converge to the optimal value.}

This implies that the proposed method ensures a monotonic enhancement in performance, ultimately converging to the optimal value.
\end{proof}

% \clearpage
\section{Details of Experiments}
\label{append:Details-of-Experiments}

\subsection{Environment Settings}

\begin{figure}[b!]
 \centering
\subcaptionbox{}
  {
\includegraphics[width=0.134\linewidth]{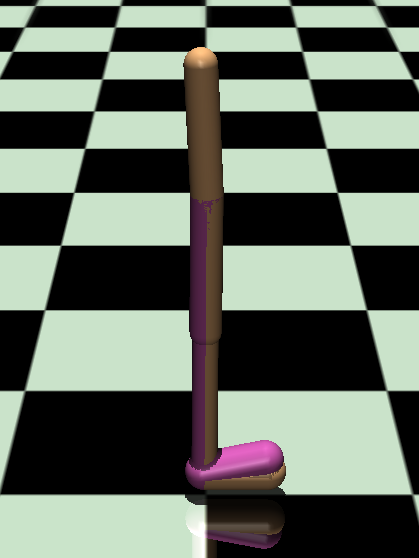}
}
 \subcaptionbox{}
  {
\includegraphics[width=0.38\linewidth]{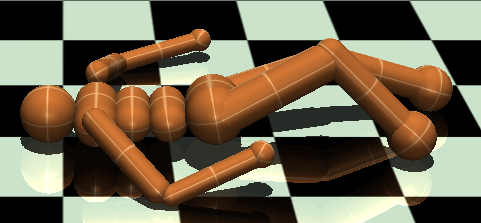}
}
 \subcaptionbox{}
  {
\includegraphics[width=0.26\linewidth]{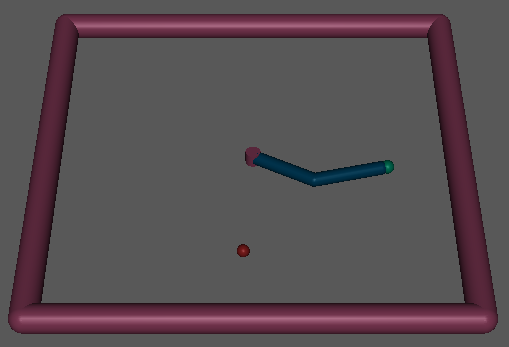}
}
 \subcaptionbox{}
  {
\includegraphics[width=0.085\linewidth]{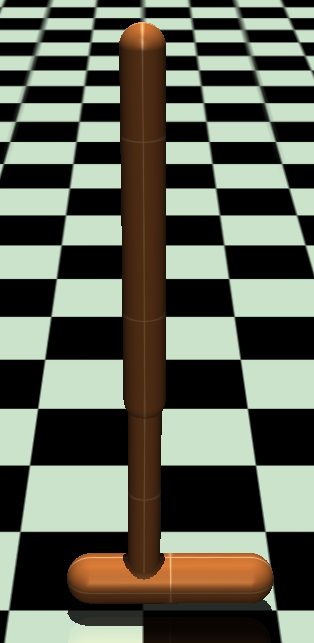}
}
 \subcaptionbox{}
  {
\includegraphics[width=0.21\linewidth]{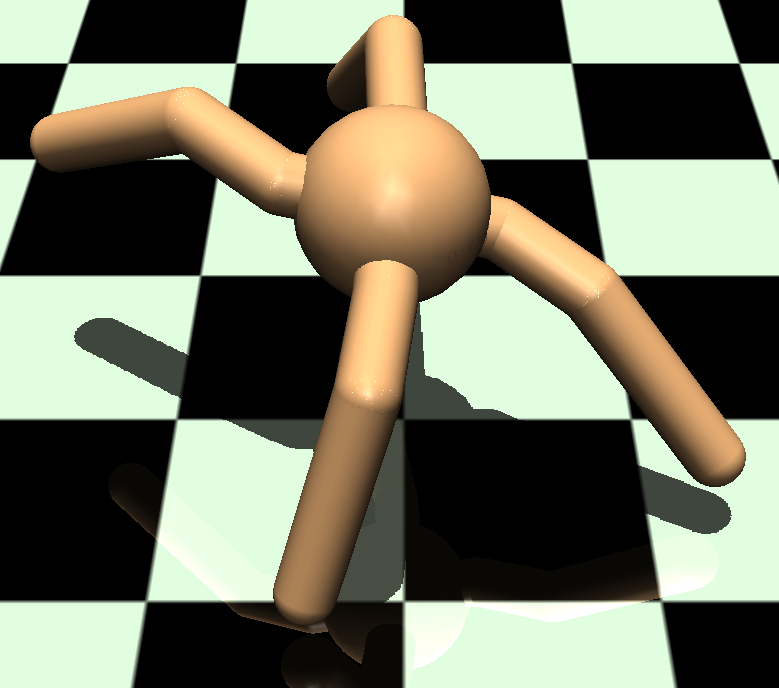}
}
 \subcaptionbox{}
  {
\includegraphics[width=0.259\linewidth]{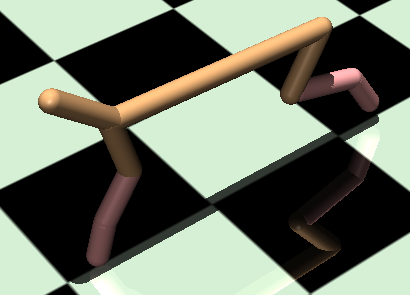}
}
 \subcaptionbox{}
  {
\includegraphics[width=0.246\linewidth]{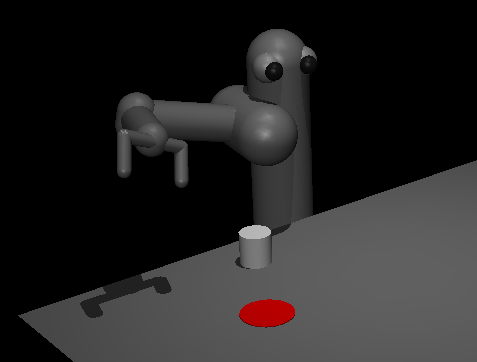}
}
 \subcaptionbox{}
  {
\includegraphics[width=0.121\linewidth]{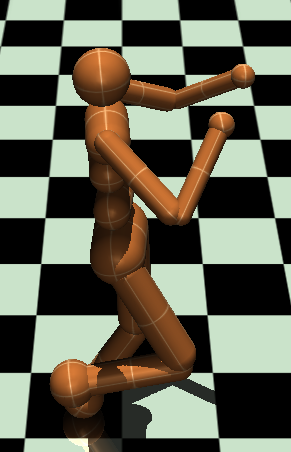}
}
    \vspace{-0pt}
 	\caption{\normalsize Safety-MuJoCo Environments: SafetyWalker-4~(a), SafetyHumanoidStandup-v4~(b), SafetyReacher-v4~(c), SafetyHopper-v4~(d), SafetyAnt-v4~(e), SafeHalfCheetah-v4~(f), SafetyPusher-v4~(g) and SafetyHumanoid-v4~(h).
 	} 
  \label{fig:safetymujoco-envs-pic-mujoco-tasks}
 \end{figure} 

\paragraph{SafetyWalker-v4.} 
As depicted in Figure \ref{fig:safetymujoco-envs-pic-mujoco-tasks} (a), the reward setting for the SafetyWalker-v4 task is illustrated in Equation (\ref{append-eq:SafetyWalker-reward-velocity}). In this equation, $X_v$ represents the forward velocity of the robot, and $\alpha_r$ signifies the reward weight. A higher reward is achieved as the robot's speed increases. The cost settings can be observed in Equation (\ref{append-eq:SafetyWalker-cost}), in which $\alpha_c$ denotes the cost weight and $a_i$ corresponds to the values of the $i^{th}$ action. A smaller cost value indicates that the robot conserves more energy.

\begin{align}
\label{append-eq:SafetyWalker-reward-velocity}
    R_{safetywalker} = \alpha_r X_v.
\end{align}

\begin{align}
\label{append-eq:SafetyWalker-cost}
    C_{safetywalker} = \alpha_c \sum_{j=0}^J a_i^2.
\end{align}

\paragraph{SafetyHumanoidStandup-v4.} 
 
As presented in Figure \ref{fig:safetymujoco-envs-pic-mujoco-tasks} (b), the reward settings for the SafetyHumanoidStandup task can be observed in Equation (\ref{append-eq:SafetyHumanoidStandup-reward}), where $Z_h$ denotes the height of the robot standing up. A higher value for the robot's height results in an increased reward value. Additionally, the cost settings are demonstrated in Equation (\ref{append-eq:SafetyHumanoidStandup-cost}), in which $a_i$ corresponds to the values of the $i^{th}$ action, and $f_{k,force}$ represents the force values of the $k^{th}$ mass.

\begin{align}
\label{append-eq:SafetyHumanoidStandup-reward}
    R_{safetyhumanoidstandup} = Z_h.
\end{align}

\begin{align}
\label{append-eq:SafetyHumanoidStandup-cost}
    C_{safetyhumanoidstandup} = 0.1 * (\sum_{j}^J a_i)^2 + 0.5e-6 * (\sum_{k=0}^K f_{k,force})^2.
\end{align}

\paragraph{SafetyReacher-v4.} 

As depicted in Figure \ref{fig:safetymujoco-envs-pic-mujoco-tasks} (c), the reward settings for the SafetyReacher task are presented in Equation (\ref{append-eq:SafetyReacher-reward}). Here, $d_{ft}$ represents the distance between the robot's fingers and the goal position, with a smaller distance yielding a higher reward. Moreover, the cost settings can be observed in Equation (\ref{append-eq:SafetyReacher-cost}), where, similarly, $a_i$ corresponds to the values of the $i^{th}$ action.

\begin{align}
\label{append-eq:SafetyReacher-reward}
    R_{safetyreacher} = -\|d_{ft}\|_2.
\end{align}

\begin{align}
\label{append-eq:SafetyReacher-cost}
    C_{safetyreacher} = \sum_{j=0}^J a_i^2.
\end{align}

\paragraph{SafetyHopper-v4.}

As depicted in Figure \ref{fig:safetymujoco-envs-pic-mujoco-tasks} (c), the reward settings for the SafetyHopper task can be observed in Equation (\ref{append-eq:SafetyHopper-reward}). These settings are similar to those of SafetyWalker, indicating that the reward value increases as the robot's speed accelerates. Furthermore, the cost settings are presented in Equation (\ref{append-eq:SafetyHopper-cost}). The first part of the cost settings resembles those of SafetyWalker, while the second part pertains to the robot's health. If the robot satisfies the health conditions, no punishment is incurred; however, if it fails to meet these conditions, a cost value of 1 is applied.

\begin{align}
\label{append-eq:SafetyHopper-reward}
    R_{safetyhopper} = \alpha_r X_v.
\end{align}

\begin{align}
\label{append-eq:SafetyHopper-cost}
    C_{safetyhopper} = \alpha_c \sum_{j=0}^J a_i^2 + C_{he}.
\end{align}

\begin{equation}
\label{append-eq:hopper-healthy}
C_{he}=\left\{
\begin{aligned}
0 & , & 3=R_{he-state} + R_z + R_{angle}, \\
1 & , & Others.
\end{aligned}
\right. 
\end{equation}

\begin{equation}
\label{append-eq:hopper-healthy-state}
R_{he-state}=\left\{
\begin{aligned}
1 & , & State_{min} \leq State_{real} \leq State_{max}, \\
0 & , & Others.
\end{aligned}
\right. 
\end{equation}

\begin{equation}
\label{append-eq:hopper-healthy-z-axis}
R_{z}=\left\{
\begin{aligned}
1 & , & Z_{min} \leq Z_{real} \leq Z_{max}, \\
0 & , & Others.
\end{aligned}
\right.
\end{equation}

\begin{equation}
\label{append-eq:hopper-healthy-angle}
R_{angle}=\left\{
\begin{aligned}
1 & , & Angle_{min} \leq Angle_{real} \leq Angle_{max}, \\
0 & , & Others.
\end{aligned}
\right.
\end{equation}

\paragraph{SafetyAnt-v4.} 

As illustrated in Figure \ref{fig:safetymujoco-envs-pic-mujoco-tasks} (e), the reward settings for the SafetyAnt task are presented in Equation (\ref{append-eq:SafetyAnt-reward}). These settings resemble those of SafetyWalker, signifying that a higher reward is achieved as the robot's speed increases. Additionally, the cost settings are similar to those of SafetyWalker. A smaller cost value indicates that the robot consumes less energy.

\begin{align}
\label{append-eq:SafetyAnt-reward}
    R_{safetyant} = X_v.
\end{align}

\begin{align}
\label{append-eq:SafetyAnt-cost}
    C_{safetyant} = \alpha_c \sum_{j=0}^J a_i^2.
\end{align}

\paragraph{SafetyHalfCheetah-v4.} 

As shown in Figure \ref{fig:safetymujoco-envs-pic-mujoco-tasks} (f), the reward settings for the SafetyHalfCheetah task are presented in Equation (\ref{append-eq:SafetyHalfCheetah-reward}). In this equation, $V_{target}$ denotes the robot's target velocity, and $V_{cheetah}$ represents the robot's current velocity. A higher reward value is achieved as the current speed approaches the target speed. The cost settings can be observed in Equation (\ref{append-eq:SafetyHalfCheetah-cost}), where $H_{cheetah}$ denotes the current height of the robot, and $H_{target}$ represents the constrained height of the robot. If the current speed exceeds the target speed by a greater margin, more cost will be incurred.

\begin{align}
\label{append-eq:SafetyHalfCheetah-reward}
    R_{safetyhalfcheetah} = -|V_{cheetah} - V_{target}|.
\end{align}

\begin{align}
\label{append-eq:SafetyHalfCheetah-cost}
    C_{safetyhalfcheetah} = |H_{cheetah} - H_{target}|.
\end{align}

\paragraph{SafetyPusher-v4.} 

As depicted in Figure \ref{fig:safetymujoco-envs-pic-mujoco-tasks} (g), the reward settings for the SafetyPusher task can be observed in Equation (\ref{append-eq:SafetyPusher-reward}). Here, $D_{object-goal}$ represents the distance between the object and the goal position, while $D_{robot-object}$ denotes the distance between the robot and the object. A higher reward is achieved when the values of these two distances are smaller. The cost settings for SafetyPusher are similar to those of the SafetyWalker task.

\begin{align}
\label{append-eq:SafetyPusher-reward}
    R_{safetypusher} = -\| \textbf{$D_{object-goal}$} \|_2 -\| \textbf{$D_{robot-object}$} \|_2.
\end{align}

\begin{align}
\label{append-eq:SafetyPusher-cost}
    C_{safetypusher} = \alpha_c \sum_{j=0}^J a_i^2.
\end{align}

\paragraph{SafetyHumanoid-v4.} 

As depicted in Figure \ref{fig:safetymujoco-envs-pic-mujoco-tasks} (g), the reward settings for the SafetyHumanoid task are similar to those of SafetyWalker. The first part of the cost settings also shares similarities with SafetyWalker. The second part of the cost settings is related to the robot's state of health, specifically concerning the robot's standing distance. If the standing distance satisfies certain conditions, no punishment is incurred; however, if it does not meet these conditions, a cost of 1 is emitted.

\begin{align}
\label{append-eq:SafetyHumanoid-reward}
    R_{safetyhumanoid} = \alpha_r X_v.
\end{align}

\begin{align}
\label{append-eq:SafetyHumanoid-cost}
    C_{safetyhumanoid} = \alpha_c \sum_{j=0}^J a_i^2 + C_{z}.
\end{align}

\begin{equation}
\label{append-eq:SafetyHumanoid-healthy-z-axis}
C_{z}=\left\{
\begin{aligned}
0 & , & Z_{min} \leq Z_{real} \leq Z_{max}, \\
1 & , & Others.
\end{aligned}
\right.
\end{equation}

% \clearpage

\subsection{Implementation Details and Additional Experiments}

In the conducted experiments, the primary hyperparameters employed can be observed in Tables \ref{table:algorithm-hyparameter-experiments-omnisafe-benchmark} and \ref{table:algorithm-hyparameter-experiments-safety-mujoco-benchmark}. Additionally, the cost limit configurations are presented in Table \ref{table:safety-cost-limit}, while the slack settings are demonstrated in Table \ref{table:slack-value-pcrpo}.  To conduct the experiments, a server equipped with 40 CPU cores (Intel® Xeon(R) Gold 5218R CPU @ 2.10GHz × 80) and 1 GTX-970 GPU (NVIDIA GeForce GTX 970/PCIe/SSE2) is utilized. The operating system running on the server is Ubuntu 18.04.

As illustrated in Table \ref{table:algorithm-hyparameter-experiments-omnisafe-benchmark}, we conducted an experiment by modifying the learning rate ($lr$) values of all baselines to $0.001$, disabling the linear learning rate decay ($linear~lr~decay$) for all baselines, and maintaining other parameters as specified in Table \ref{table:algorithm-hyparameter-experiments-omnisafe-benchmark}. The results of this experiment are depicted in Figure \ref{fig:compared-omnisafe-Ant-hopper-swimmer-baselines-same-parameters}, where our algorithm's performance is significantly superior to that of the SOTA baselines. 
In particular, as illustrated in Figures \ref{fig:compared-omnisafe-Ant-hopper-swimmer-baselines-same-parameters} (a-d), we conducted experiments on the SafetyAntVelocity tasks with identical hyperparameters. To better showcase the results, Figures \ref{fig:compared-omnisafe-Ant-hopper-swimmer-baselines-same-parameters} (a) and (b) display the outcomes of PCRPO and PCPO, while Figures \ref{fig:compared-omnisafe-Ant-hopper-swimmer-baselines-same-parameters} (c) and (d) present the results of CUP and PPOLag. These findings reveal that CUP and PPOlag do not perform well under the same conditions. Figures \ref{fig:compared-omnisafe-Ant-hopper-swimmer-baselines-same-parameters} (e) and (f) depict the results of all algorithms on the SafetyHopperVelocity tasks. The outcomes demonstrate that our algorithm not only achieves superior reward performance and faster convergence compared to all SOTA baselines, but also ensures safety, which the SOTA baselines fail to guarantee. Similar observations can be made for the SafetySwimmerVelocity tasks, as shown in Figures \ref{fig:compared-omnisafe-Ant-hopper-swimmer-baselines-same-parameters} (g) and (h).

It is worth noting that the parameters of the other baselines were fine-tuned by \cite{ji2023omnisafe}, which may represent the best performance achievable by the SOTA baselines, and our method can outperform all the baselines in terms of reward, safety performance and convergence. By thoroughly evaluating the effectiveness and efficiency of our proposed method in comparison to the baseline algorithms, we aimed to provide valuable insights into their respective strengths and weaknesses, as well as identify potential areas for improvement and future research directions.

\begin{table}[!htbp]
 \renewcommand{\arraystretch}{1.2}
  \centering
  \begin{threeparttable}
    \begin{tabular}{c|ccccc}
    \toprule
    \diagbox{parameters}{values}{algorithms} & CUP & PCPO& PPOLag&PCRPO\\
    \midrule
        device & cpu  & cpu & cpu & cpu \\ 
        torch threads & 1  & 1 & 1 & 1   \\
        vector env nums & 1  & 1 & 1 & 1  \\    
        parallel & 1  & 1 & 1 & 1\\     
        total steps & 10 million  & 10 million & 10 million& 10 million \\ 
        steps per epoch & 20000  & 20000 & 20000 & 20000 \\
        update iters & 40  & 10 & 40 & 10 \\
        batch size & 64  & 128 & 64 & 128 \\
        target kl & 0.01  & 0.01 & 0.02 & 0.01 \\
        entropy coef & 0  & 0 & 0 & 0 \\
        reward normalize & False  & False & False & False \\
        cost normalize & False  & False & False & False \\
        obs normalize & True  & True & True & True \\
        % kl early stop & True  & False & True & False \\
        use max grad norm & True  & True & True & True \\
        max grad norm & 40  & 40 & 40 & 40 \\
        use critic norm & True  & True & True & True \\
        critic norm coef & 0.001  & 0.001 & 0.001 & 0.001 \\
        gamma & 0.99  & 0.99 & 0.99 & 0.99 \\
        cost gamma & 0.99  & 0.99 & 0.99 & 0.99 \\
        lam & 0.95  & 0.95 & 0.95 & 0.95 \\
        lam c & 0.95  & 0.95 & 0.95 & 0.95 \\
        clip & 0.2  & \textbackslash & 0.2 & \textbackslash \\
        adv estimation method & gae  & gae & gae & gae \\
        standardized rew adv & True  & True & True & True \\
        standardized cost adv & True  & True & True & True \\
        penalty coef & 0  & 0 & 0 & 0 \\
        cg damping & \textbackslash  & 0.1 & \textbackslash & 0.1 \\
        cg iters & \textbackslash  & 15 & \textbackslash & 15 \\
        % fvp sample freq & 1  & \textbackslash & \textbackslash & \textbackslash \\
        actor type & gaussian learning  & gaussian learning & gaussian learning & gaussian learning \\
        hidden sizes & [64, 64]  & [64, 64] & [64, 64] & [64, 64] \\
        activation & tanh  & tanh & tanh & tanh \\
        lr & 0.0003  & 0.001 & 0.0003 & 0.001 \\
        lagrangian multiplier init & 0.001  & \textbackslash & 0.001 & \textbackslash \\
        lambda lr & 0.035  & \textbackslash & 0.035 & \textbackslash \\
    \bottomrule
    \end{tabular}    
    \end{threeparttable}
    \vspace{6pt}
\caption{Key hyparameters used in Omnisafe. }
\label{table:algorithm-hyparameter-experiments-omnisafe-benchmark}
\end{table}

\begin{figure}[htbp!]
 \centering
 \subcaptionbox{}
 {
\includegraphics[width=0.22\linewidth]{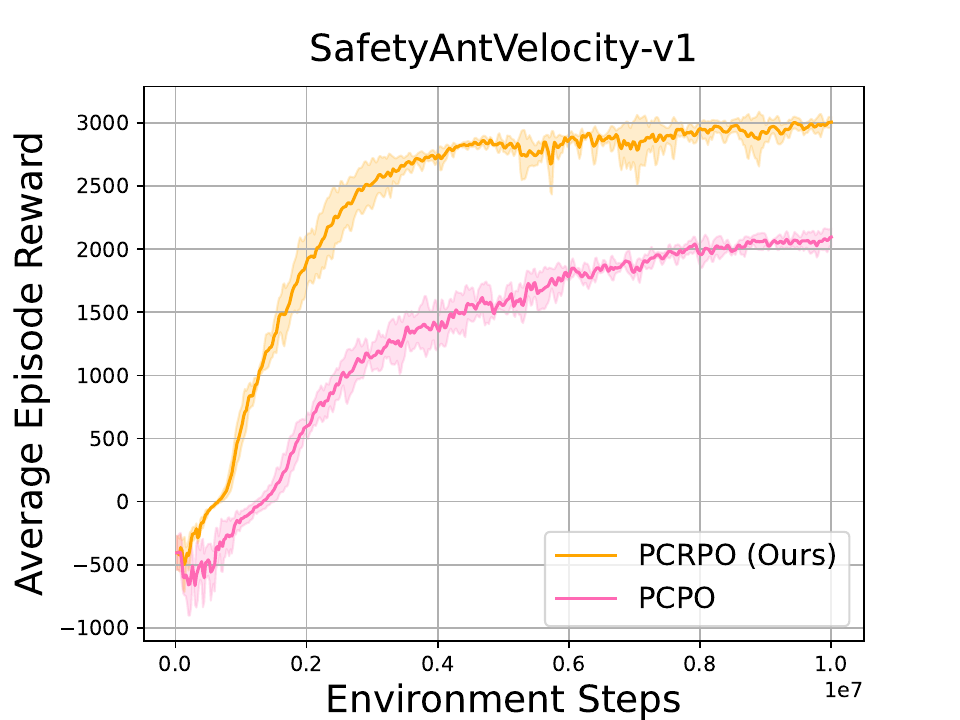}
}
 \subcaptionbox{}
 {
\includegraphics[width=0.22\linewidth]{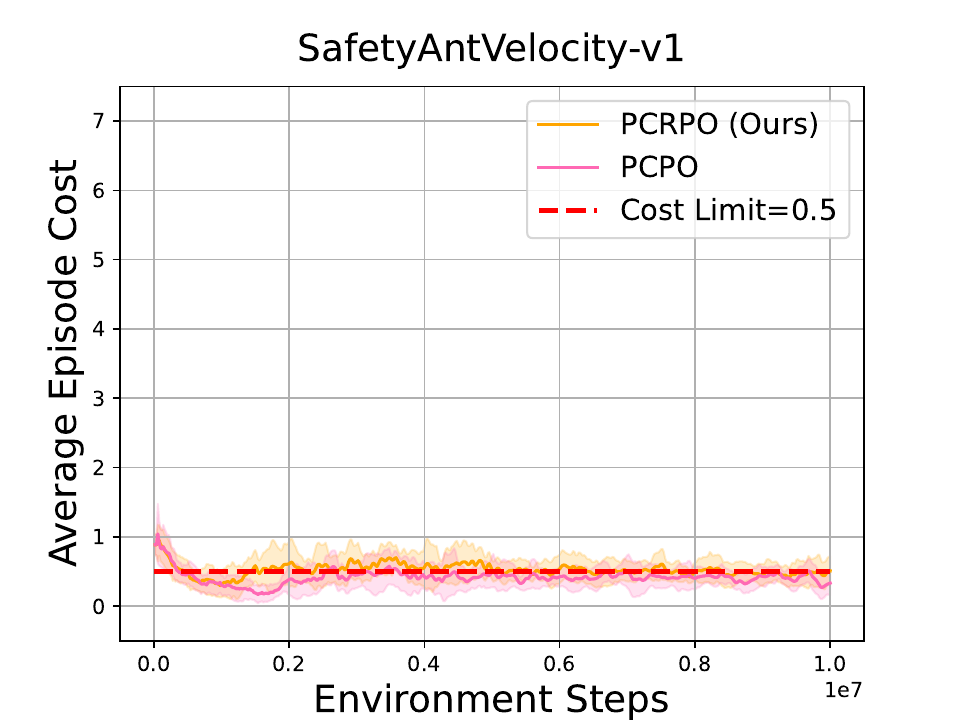}
}
\subcaptionbox{}
 {
\includegraphics[width=0.22\linewidth]{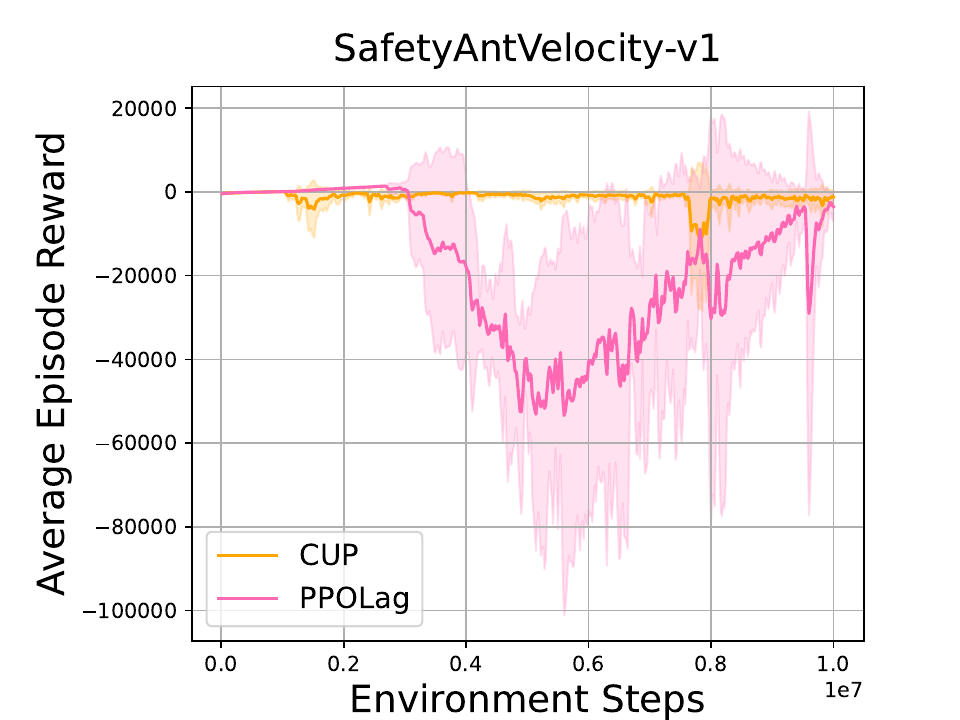}
}
 \subcaptionbox{}
 {
\includegraphics[width=0.22\linewidth]{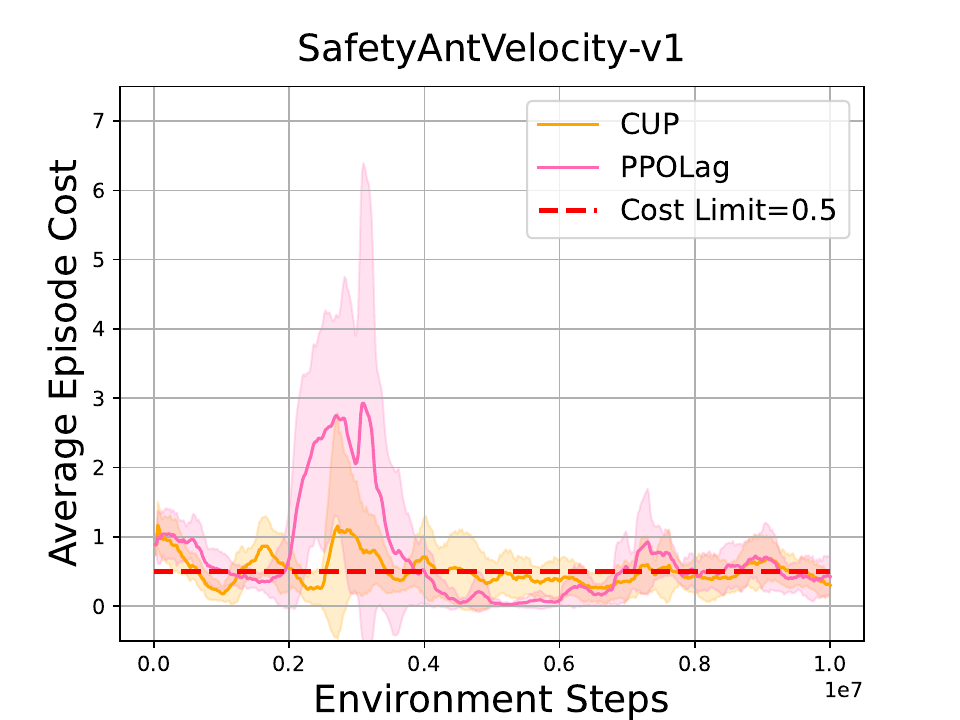}
}
 \subcaptionbox{}
 {
\includegraphics[width=0.22\linewidth]{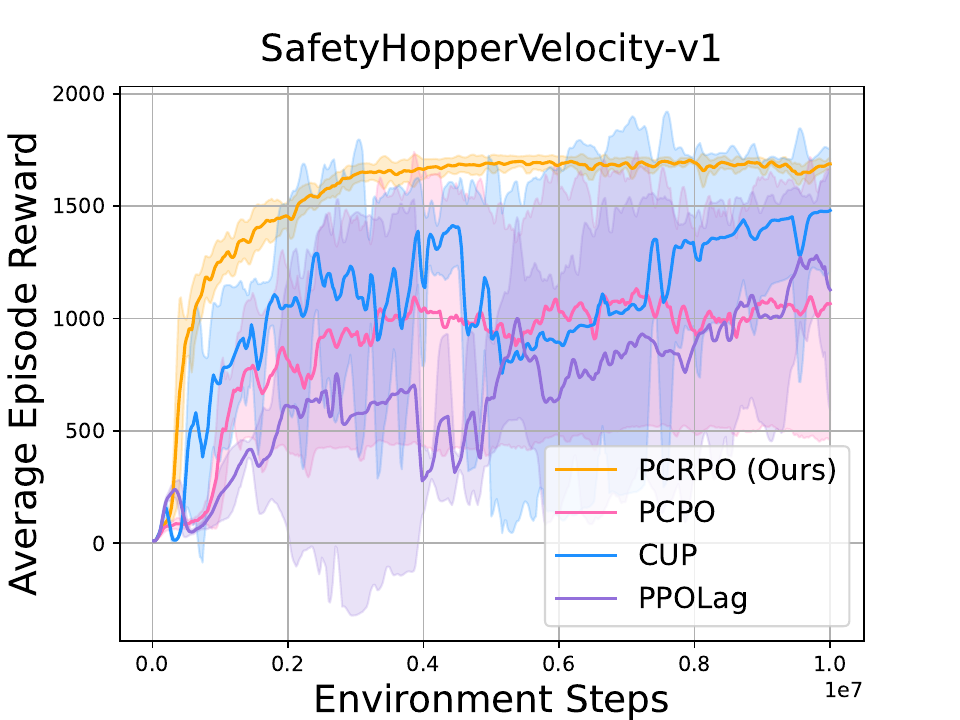}
}
 \subcaptionbox{}
 {
\includegraphics[width=0.22\linewidth]{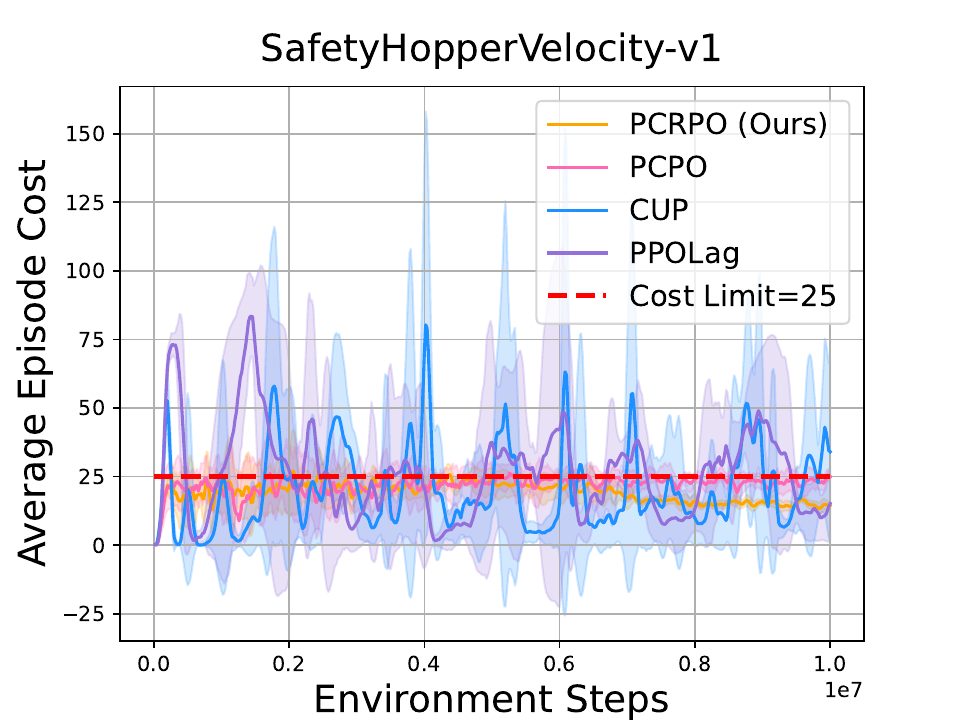}
}
 \subcaptionbox{}
 {
\includegraphics[width=0.22\linewidth]{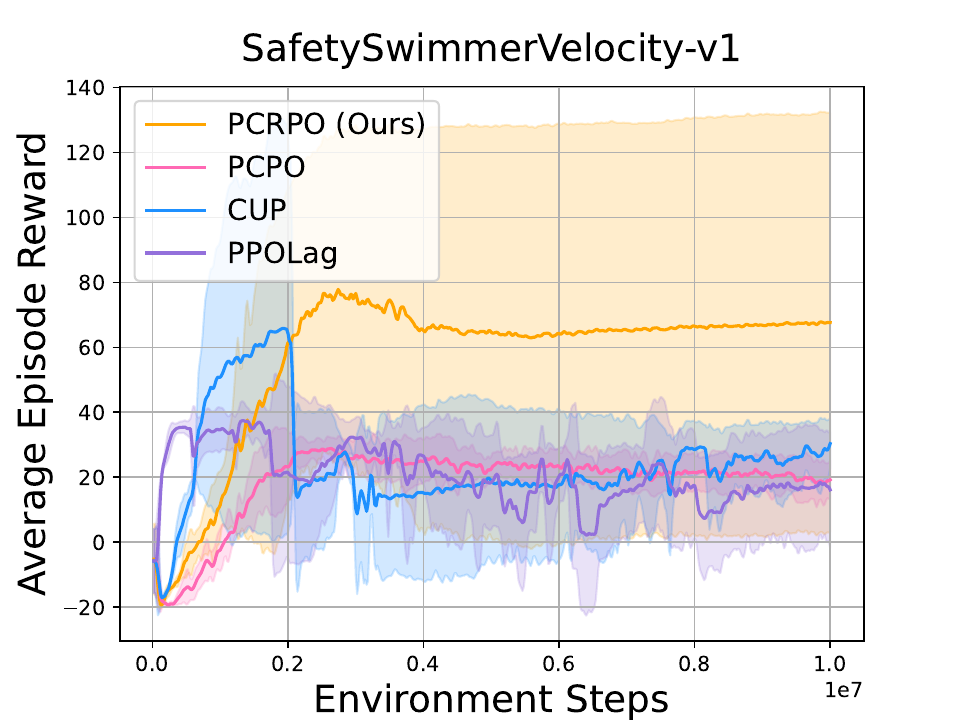}
}
 \subcaptionbox{}
 {
\includegraphics[width=0.22\linewidth]{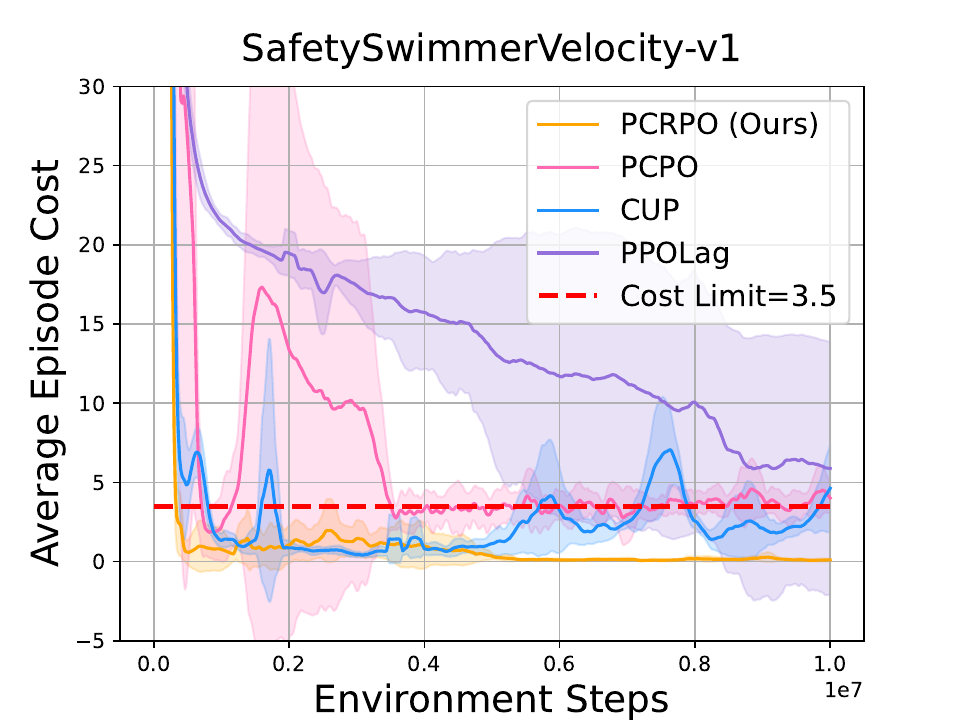}
}
\vspace{-5pt}
 	\caption{\normalsize 
  In the additional experiments, the performance of the proposed method was compared with several baseline algorithms, including PCPO \cite{yang2020projection}, CUP \cite{yang2022constrained}, and PPOLag \cite{ji2023omnisafe}. The comparison was conducted on three safety-constrained tasks: SafetyAntVelocity-v1, SafetyHopperVelocity-v1, and SafetyHopperVelocity-v1.
 	} 
  \label{fig:compared-omnisafe-Ant-hopper-swimmer-baselines-same-parameters}
\vspace{-15pt}
 \end{figure}

\begin{table}[!htbp]
 \renewcommand{\arraystretch}{1.2}
  \centering
  \begin{threeparttable}
    \begin{tabular}{cc|cc}
    \toprule
    Parameters & value & Parameters & value \\
    \midrule
    gamma & 0.995             &       tau & 0.97    \\                 l2-reg & 1e-3 & cost kl &  0.05   \\
           damping & 1e-1          &  batch-size & 16000  \\    
           epoch & 500          &  episode length & 1000  \\     
           grad-c & 0.5          & neural network  & MLP  \\ 
           hidden layer dim & 64          & accept ratio  & 0.1  \\
           energy weight & 1.0          & forward reward weight  & 1.0  \\
    \bottomrule
    \end{tabular}    
    \end{threeparttable}
    \vspace{6pt}
\caption{Key hyparameters used in Safety-MuJoCo Benchmark. }
\label{table:algorithm-hyparameter-experiments-safety-mujoco-benchmark}
\end{table}

\begin{table}[!htbp]
 \renewcommand{\arraystretch}{1.2}
  \centering
  \begin{threeparttable}
    \begin{tabular}{cc}
    \toprule
    Slack settings & KL values \\
    \midrule
    2SR & 0.01 for reward (Not violation), 0.05 for both reward and cost (Violations). \\                 2SC & 0.05 for all rewards.     \\
           3SC-F, 3SC-G, 3SR-F,  3SR-G& 0.01 for all Rewards.            \\ 
           4S-G-V0, 4S-G-V1, 4S-F-V0, 4S-F-V1 & 0.01 for all rewards.              \\           
    \bottomrule
    \end{tabular}    
    \end{threeparttable}
    \vspace{6pt}
\caption{KL settings in experiments. }
\label{table:kl-for-reward-cost}
\end{table}

\begin{table}[!htbp]
 \renewcommand{\arraystretch}{1.2}
  \centering
  \begin{threeparttable}
    \begin{tabular}{cc}
    \toprule
    Environment & Cost limit \\
    \midrule
    SafetyWalker-v4 & 40     \\                 SafetyHumanoidStandup-v4 & 1200    \\
           SafetyReacher-v4 & 40            \\ 
           SafetyAntVelocity-v1 & 0.5             \\ 
           SafetyHopperVelocity-v1 & 25            \\
           SafetySwimmerVelocity-v1 & 3.5    \\
           SafetySwimmerVelocity-v1 (ablation of cost limit ) & 0.08    \\SafetyWalker-v4 (ablation of slack settings ) &  40      \\
           SafetyHumanoidStandup-v4 (ablation of gradient manipulation ) &  1200   \\
    \bottomrule
    \end{tabular}    
    \end{threeparttable}
    \vspace{6pt}
\caption{Cost limit settings in experiments. }
\label{table:safety-cost-limit}
\end{table}

\begin{table}[!htbp]
 \renewcommand{\arraystretch}{1.2}
  \centering
  \begin{threeparttable}
    \begin{tabular}{cc}
    \toprule
    Environment & Slack value \\
    \midrule
    SafetyWalker-v4 & $h^+=5, h^-=-5$     \\                 SafetyHumanoidStandup-v4 & $h^+=300, h^-=-300$    \\
           SafetyReacher-v4 & $h^+ =0, h^- \rightarrow \infty$            \\ 
           SafetyAntVelocity-v1 & $h^+=0.25, h^-=-0.25$             \\ 
           SafetyHopperVelocity-v1 & $h^+=0, h^-=-9$             \\
           SafetySwimmerVelocity-v1 & $h^+ =0, h^- \rightarrow \infty$   \\
           SafetySwimmerVelocity-v1 (ablation of cost limit ) & $h^+ =0.04, h^- =-0.04$   \\SafetyWalker-v4 (ablation of slack settings ) &  $h^+ \rightarrow +\infty, h^- =0$ (PCRPO-2SR)      
           \\SafetyWalker-v4 (ablation of slack settings ) &  $h^+ =20, h^- =0$ (PCRPO-3SR-G)   
           \\SafetyWalker-v4 (ablation of slack settings ) &  $h^+ =20, h^- =-20$ (PCRPO-4S-F \& PCRPO-4S-G)      \\
           SafetyHumanoidStandup-v4 (ablation of gradient manipulation methods) &  $h^+=300, h^-=-300$ (PCRPO \& SCRPO)      \\
    \bottomrule
    \end{tabular}    
    \end{threeparttable}
    \vspace{6pt}
\caption{Slack settings in experiments. }
\label{table:slack-value-pcrpo}
\end{table}

\end{document}